\documentclass[letterpaper]{article} 
\usepackage{aaai23}  
\usepackage{times}  
\usepackage{helvet}  
\usepackage{courier}  
\usepackage[hyphens]{url}  
\usepackage{graphicx} 
\usepackage{natbib}  
\usepackage{caption} 
\usepackage{algorithm}

\usepackage{newfloat}
\usepackage{listings}
\DeclareCaptionStyle{ruled}{labelfont=normalfont,labelsep=colon,strut=off} 
\floatstyle{ruled}
\newfloat{listing}{tb}{lst}{}
\floatname{listing}{Listing}
\title{Optimizing Learning Rate Schedules \\ for Iterative Pruning of Deep Neural Networks}

\author {
	Shiyu Liu,
	Rohan Ghosh, 
	John Tan Chong Min,
	Mehul Motani
}
\affiliations {
	Department of Electrical and Computer Engineering \\
	School of Engineering, National University of Singapore \\
	shiyu\_liu@u.nus.edu, rghosh92@gmail.com, e0441892@u.nus.edu, motani@nus.edu.sg
}

\usepackage{graphicx}

\usepackage{tikz}
\usepackage{xcolor}
\usepackage{comment}
\usepackage{amsthm, amssymb} 
\usepackage{color}

\theoremstyle{definition}
\newtheorem{definition}{Definition}
\newtheorem{theorem}{Theorem}
\newtheorem{corollary}{Corollary}
\newtheorem{remark}{Remark}

\usepackage[accsupp]{axessibility}

\usepackage{graphicx}
\usepackage{amsmath}
\usepackage{amssymb}
\usepackage{booktabs}
\usepackage{multirow}
\DeclareMathOperator\erfi{erf^{-1}}

\usepackage{algorithm}
\usepackage[noend]{algpseudocode}
\usepackage{amsmath}
\usepackage{amsfonts}
\usepackage{amssymb}

\makeatletter
\@namedef{ver@everyshi.sty}{}
\makeatother
\usepackage{tikz}

\usepackage{scalerel}
\usepackage{enumitem}
\usepackage{graphicx}
\graphicspath{{img/}}

\usepackage[export]{adjustbox}

\usepackage{lscape}

\usepackage{xcolor}
\usepackage{amsmath}
\usepackage{booktabs}

\usepackage{tikz}
\usetikzlibrary{backgrounds}
\usetikzlibrary[
positioning,
fit,
arrows.meta,
calc,
patterns,
shapes.arrows,
]
\usetikzlibrary{math, patterns}

\usepackage{newfloat}
\usepackage{listings}
\makeatletter
\newcommand*{\rom}[1]{\expandafter\@slowromancap\romannumeral #1@}
\makeatother

\definecolor{blue}{rgb}{0.4, 0.4, 1}
\definecolor{red}{rgb}{1, 0.2, 0.2}
\usepackage{balance}
\usepackage{breqn}

\begin{document}

\maketitle

\begin{abstract}
The importance of learning rate (LR) schedules on network pruning has been observed in a few recent works. As an example, {\it Frankle and Carbin (2019)} highlighted that winning tickets (i.e., accuracy preserving subnetworks) can not be found without applying a LR warmup schedule and {\it Renda, Frankle and Carbin (2020)} demonstrated that rewinding the LR to its initial state at the end of each pruning cycle improves performance. In this paper, we go one step further by first providing a theoretical justification for the surprising effect of LR schedules. Next, we propose a LR schedule for network pruning called SILO, which stands for S-shaped Improved Learning rate Optimization. The advantages of SILO over existing state-of-the-art (SOTA) LR schedules are two-fold: (i) SILO has a strong theoretical motivation and dynamically adjusts the LR during pruning to improve generalization. Specifically, SILO increases the LR upper bound ($\texttt{max\_lr}$) in an S-shape. This leads to an improvement of 2\% - 4\% in extensive experiments with various types of networks (e.g., Vision Transformers, ResNet) on popular datasets such as ImageNet, CIFAR-10/100. (ii) In addition to the strong theoretical motivation, SILO is empirically optimal in the sense of matching an Oracle, which exhaustively searches for the optimal value of $\texttt{max\_lr}$ via grid search. We find that SILO is able to precisely adjust the value of $\texttt{max\_lr}$ to be within the Oracle optimized interval, resulting in performance competitive with the Oracle with significantly lower complexity.
\end{abstract}

\section{Introduction}
\vspace{-1mm}
Network pruning is the process of simplifying neural networks by pruning weights, filters or neurons. \cite{lecun1990optimal,han2015learning}. Several state-of-the-art pruning methods \cite{renda2020comparing,frankle2018lottery} have demonstrated that a significant quantity of parameters can be removed without sacrificing accuracy. This greatly reduces the resource demand of neural networks, such as storage requirements and energy consumption \cite{He_2020_CVPR,Wang_2021_CVPR}. 

The inspiring performance of pruning methods hinges on a key factor - Learning Rate (LR). Specifically, \cite{frankle2018lottery} proposed the Lottery Ticket Hypothesis and demonstrated that the winning tickets (i.e., the pruned network that can train in isolation to full accuracy) cannot be found without applying a LR warmup schedule. In a follow-up work, \cite{renda2020comparing} proposed LR rewinding which rewinds the LR schedule to its initial state during iterative pruning and demonstrated that it can outperform standard fine-tuning. 
In summary, the results in both works suggest that, besides the pruning metric, LR also plays an important role in network pruning and could be another key to improving the pruning performance.

In this paper, we take existing studies one step further and aim to optimize the choice of LR for iterative network pruning. We explore a new perspective on adapting the LR schedule to improve the iterative pruning performance. Our contributions to network pruning are as follows.

\begin{enumerate}[noitemsep , leftmargin=5mm, topsep=0pt]
	\item {\bf Motivation and Theoretical Study. }We explore the optimal choice of LR during pruning and find that the distribution of weight gradients tends to become narrower during pruning, suggesting that a larger value of LR should be used to retrain the pruned network. This finding is further verified by our theoretical development. More importantly, our theoretical results suggest that the optimal increasing trajectory of LR should follow an S-shape.
	
	\item {\bf Proposed SILO. }We propose a novel LR schedule for network pruning called SILO, which stands for S-Shaped Improved Learning rate Optimization. Motivated by our theoretical development, SILO precisely adjusts the LR by increasing the LR upper bound ($\texttt{max\_lr}$) in an S-shape. We highlight that SILO is method agnostic and works well with numerous pruning methods.
	
	\item {\bf Experiments. }We compare SILO to four LR schedule benchmarks via both classical and state-of-the-art (SOTA) pruning methods. We observe that SILO outperforms LR schedule benchmarks, leading to an improvement of 2\% - 4\% in extensive experiments with SOTA networks (e.g., Vision Transformer \cite{dosovitskiy2020image}, ResNet \cite{resnet18} \& VGG \cite{vgg16}) on large-scale datasets such as ImageNet \cite{deng2009imagenet} and popular datasets such as CIFAR-10/100 \cite{cifar10}.
	
	\item {\bf Comparison to Oracle. }We examine the optimality of SILO by comparing it to an Oracle which exhaustively searches for the optimal value of $\texttt{max\_lr}$ via grid search. We find that SILO is able to precisely adjust $\texttt{max\_lr}$ to be within the Oracle's optimized $\texttt{max\_lr}$ interval at each pruning cycle, resulting in performance competitive with the Oracle, but with significantly lower complexity.
\end{enumerate}

\section{Background}

\subsection{Prior Works on Network Pruning}
\label{sec2.1}

Network pruning is an established idea dating back to 1990 \cite{lecun1990optimal}. The motivation is that networks tend to be overparameterized and redundant weights can be removed with a negligible loss in accuracy
\cite{arora2018optimization,allenzhu2018convergence,denil2013predicting}. Given a trained network, one {\bf pruning cycle} consists of three steps as follows.
\begin{enumerate}[noitemsep,leftmargin=5mm, topsep=1pt]
	\item Prune the network according to certain heuristics.
	\item Freeze pruned parameters as zero.
	\item Retrain the pruned network to recover the accuracy.
\end{enumerate}
Repeating the pruning cycle multiple times until the target sparsity or accuracy is met is known as {\bf iterative pruning}. Doing so often results in better performance than {\bf one-shot pruning} (i.e., perform only one pruning cycle) \cite{han2015learning,frankle2018lottery,li2016pruning}. There are two types of network pruning - unstructured pruning and structured pruning - which will be discussed in detail below.

{\bf Unstructured Pruning} removes individual weights according to certain heuristics such as magnitude \cite{han2015learning} or  gradient \cite{hassibi1993second,lee2018snip,xiao2019autoprune,theis2018faster}. Examples are \cite{lecun1990optimal}, which performed pruning based on the Hessian Matrix, and \cite{theis2018faster}, which used Fisher information to approximate the Hessian Matrix.
Similarly, \cite{han2015learning} removed weights with the smallest magnitude and this approach was further incorporated with the three-stage iterative pruning pipeline in \cite{han2015deep}.

{\bf Structured Pruning} involves pruning weights in groups, neurons, channels or filters \cite{He2019median,luo2017thinet,tan2020dropnet,wang2020dynamic,lin2020hrank}. Examples are \cite{hu2016network}, which removed neurons with high average zero output ratio, and \cite{li2016pruning}, which pruned neurons with the lowest absolute summation values of incoming weights. More recently, \cite{yu2018nisp} proposed the neuron importance score propagation algorithm to evaluate the importance of network structures. \cite{molchanov2019importance} used Taylor expansions to approximate a filter's contribution to the final loss and \cite{Wang_2020_CVPR} optimized the neural network architecture, pruning policy, and quantization policy together in a joint manner.

{\bf Other Works. } In addition to works mentioned above, several other works also share some deeper insights in network pruning \cite{liu2018rethinking,zhu2017prune,liu2018darts,wang2020pruning}. For example, \cite{liu2018darts} demonstrated that training-from-scratch on the right sparse architecture yields better results than pruning from pre-trained models. Similarly, \cite{wang2020pruning} suggested that the fully-trained network could reduce the search space for the pruned structure. More recently, \cite{luo2020neural} addressed the issue of pruning residual connections with limited data and \cite{ye2020good} theoretically proved the existence of small subnetworks with lower loss than the unpruned network. One milestone paper \cite{frankle2018lottery} pointed out that re-initializing with the original parameters (known as weight rewinding) plays an important role in pruning and helps to further prune the network with negligible loss in accuracy. Some follow-on works \cite{zhou2019deconstructing,renda2020comparing,malach2020proving} investigated this phenomenon more precisely and applied this method in other fields (e.g., transfer learning \cite{mehta2019sparse} and natural language processing \cite{yu2019playing}).

\subsection{The Important Role of Learning Rate}
\label{sec2.2}
Several recent works \cite{renda2020comparing,frankle2018lottery} have noticed the important role of LR in network pruning. For example, {\it Frankle and Carbin}  \cite{frankle2018lottery} demonstrated that training VGG-19 \cite{vgg16} with a LR warmup schedule (i.e., increase LR to 1$\texttt{e}$-1 and decrease it to 1$\texttt{e}$-3) and a constant LR of 1$\texttt{e}$-2 results in comparable accuracy for the unpruned network. However, as the network is iteratively pruned, the LR warmup schedule leads to a higher accuracy (see Fig.7 in \cite{frankle2018lottery}).   
In a follow-up work,  {\it Renda, Frankle and Carbin} \cite{renda2020comparing} further investigated this phenomenon and proposed a retraining technique called LR rewinding which can always outperform the standard retraining technique called fine-tuning \cite{han2015learning}. The difference is that fine-tuning trains the unpruned network with a LR warmup schedule, and retrains the pruned network with a constant LR (i.e., the final LR of the schedule) in subsequent pruning cycles \cite{liu2018rethinking}. LR rewinding retrains the pruned network by rewinding the LR warmup schedule to its initial state, namely that LR rewinding uses the same schedule for every pruning cycle. As an example, they demonstrated that retraining the pruned ResNet-50 using LR rewinding yields higher accuracy than fine-tuning (see Figs.1 \& 2 in \cite{renda2020comparing}). In \cite{shiwei2021}, the authors also suggest that when pruning happens during the training phase with a large LR, models can easily recover from pruning than using a smaller LR.
Overall, The results in these works suggest that, besides the pruning metric, LR also plays an important role in network pruning and could be another key to improving network pruning.

\begin{figure*}[!t]
	\hspace{0mm}\begin{minipage}{0.5\textwidth}
		\begin{center}
		\includegraphics[width=0.85\linewidth]{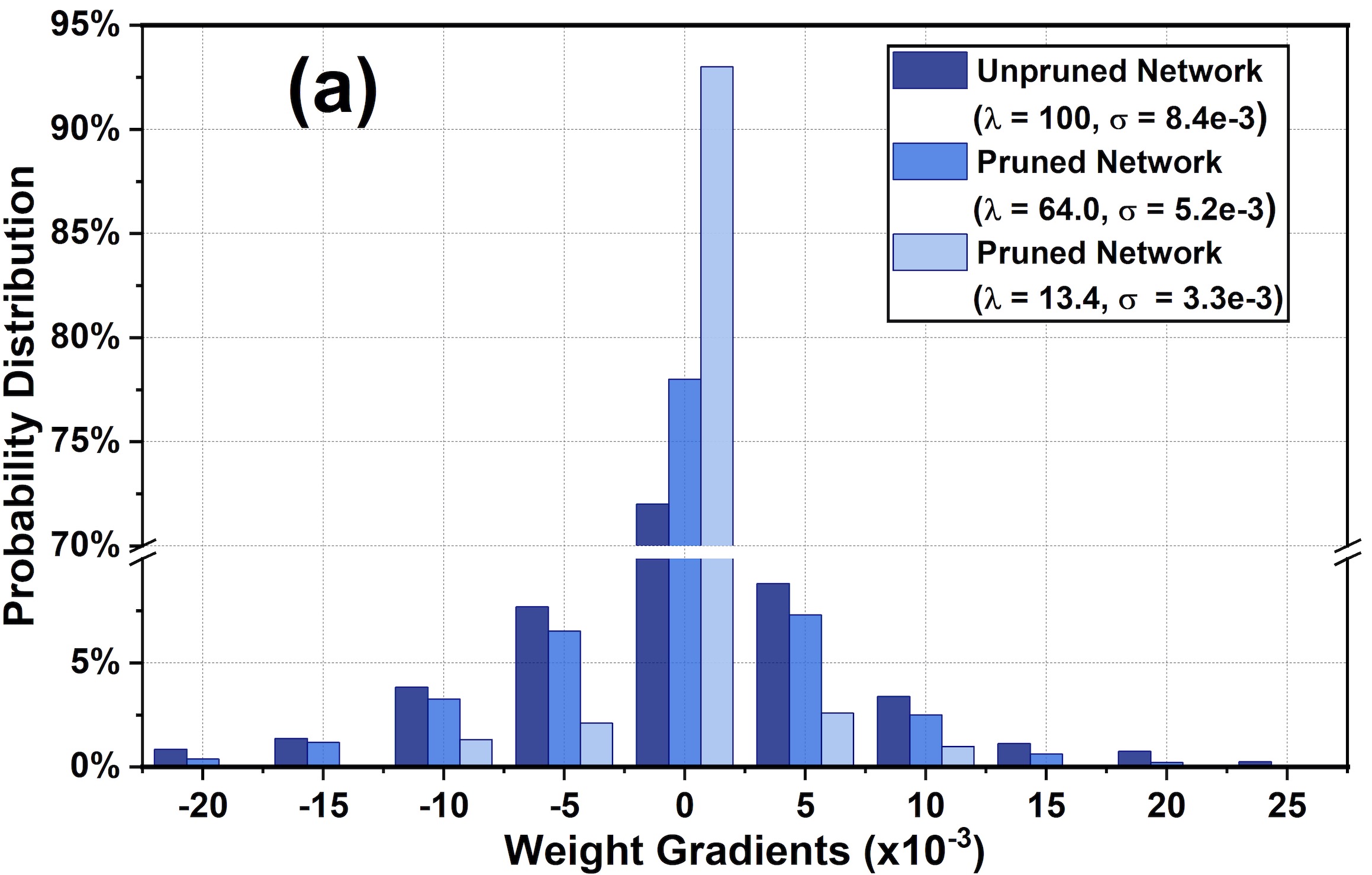}
		\end{center}
	\end{minipage}%
	\hspace{0mm}\begin{minipage}{0.5\textwidth}
		\begin{center}
		\includegraphics[width=0.85\linewidth]{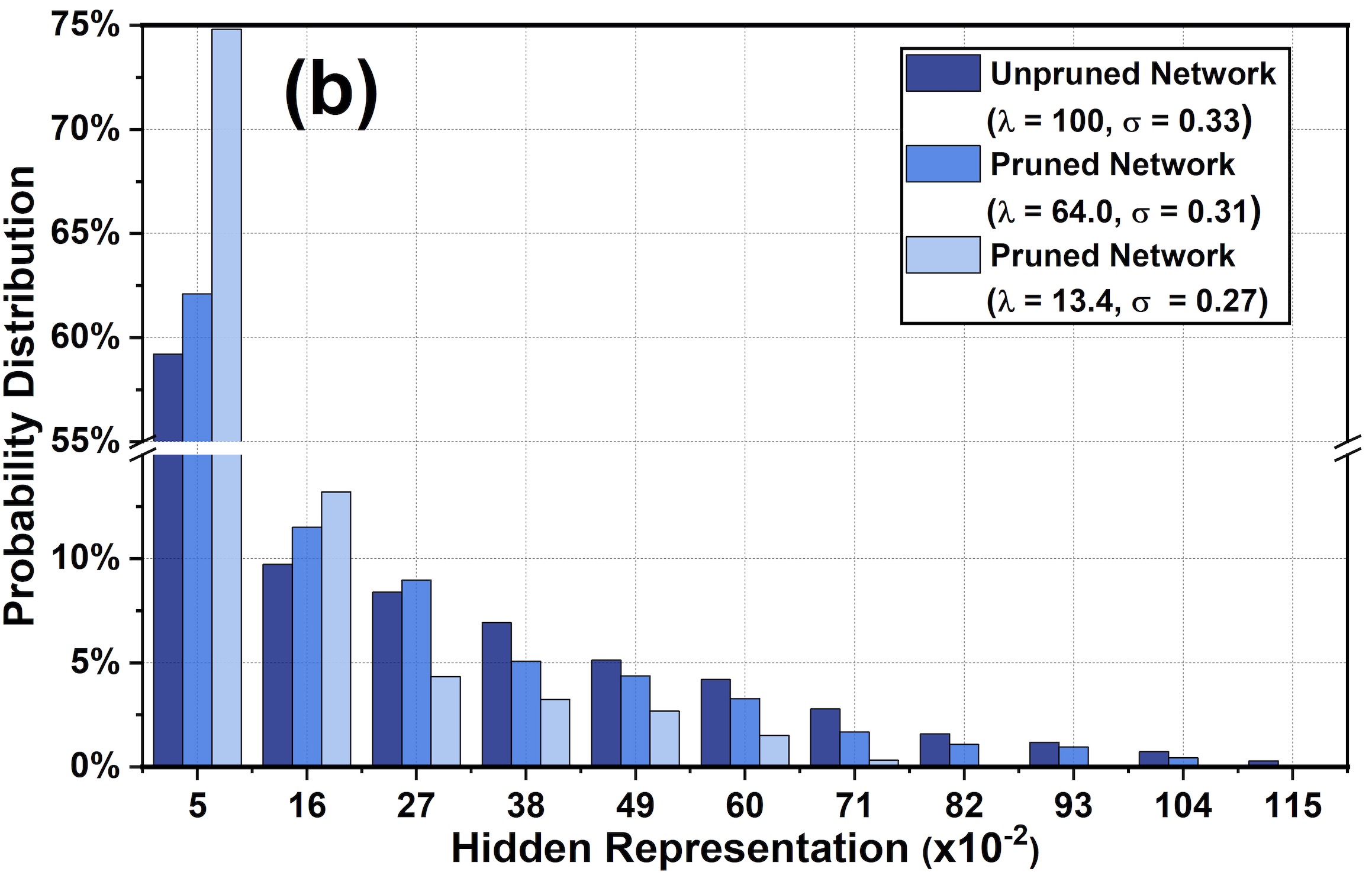}
		\end{center}
	\end{minipage}
	\vspace{-3mm}
	\caption{(a) The distribution of all weight gradients when iteratively pruning a fully connected ReLU network using global magnitude (i.e., prune weights with the lowest magnitude), where $\lambda$ is the percent of weights remaining and $\sigma$ is the standard deviation of the distribution. (b) The corresponding distribution of hidden representations. Please note that there is a line breaker in the vertical axis. Similar findings can be observed using popular networks and pruning methods (see {\bf Appendix}).} 
	\label{mlp_no_norm}
	\vspace{-3mm}
\end{figure*}

{\bf Our work:} In this paper, we explore a new perspective on adapting the LR schedule to improve the iterative pruning performance of ReLU-based networks. The proposed LR schedule is method agnostic and can work well with numerous pruning methods. We mainly focus on iterative pruning of ReLU-based networks for two reasons: (i) Iterative pruning tends to provide better pruning performance than one-shot pruning as reported in the literature \cite{frankle2018lottery,renda2020comparing}. (ii) ReLU has been widely used in many classical neural networks (e.g., ResNet, VGG, DenseNet) which have achieved outstanding performance in various tasks (e.g., image classification, object detection) \cite{resnet18,vgg16}.

\section{A New Insight on Network Pruning}
\label{sec3}
In Section \ref{3.1}, we first provide a new insight in network pruning using experiments. Next, in Section \ref{sec3.2}, we provide a theoretical justification for our observed new insight and present some relevant theoretical results.

\subsection{Weight Gradients during Iterative Pruning}
\label{3.1}

{\bf (1) Experiment Setup.} To exclude the influence of other factors, we start from a simple fully connected ReLU-based network with three hidden layers of 256 neurons each (results of other popular networks are summarized later). We train the network using the training dataset of CIFAR-10 via SGD \cite{ruder2016overview} (momentum = 0.9 and a weight decay of 1\texttt{e}-4) with a batch size of 128 for 500 epochs. All hyperparameters are tuned for performance via grid search (e.g., LR from 1\texttt{e}-4 to 1\texttt{e}-2). We apply the global magnitude \cite{han2015learning} (i.e., remove weights with the smallest magnitude anywhere in the network) with a pruning rate of 0.2 (i.e., prune 20\% of the remaining parameters) to iteratively prune the network for 10 pruning cycles and plot the distribution of all weight gradients when the network converges in Fig. \ref{mlp_no_norm}(a), where $\lambda$ is the percent of weights remaining.
In Fig. \ref{mlp_no_norm}(a), there are 10 visible bins estimated by the Sturges' Rule \cite{scott2009sturges} and each bin consists of three values (i.e., the probability distribution of networks). The edge values range from -0.022 to 0.027 with a bin width of 0.004.

{\bf (2) Experiment Results.} In Fig. \ref{mlp_no_norm}(a), we observe that the distribution of weight gradients tends to become narrower, i.e., the standard deviation of weight gradients $\sigma$ reduces from 8.4\texttt{e}-3 to 3.3\texttt{e}-3 when the network is iteratively pruned to $\lambda = 13.4$. As an example, the unpruned network ($\lambda$ = 100) has more than 7\% of weight gradients with values greater than 0.008 (rightmost 4 bars) or less than -0.012 (leftmost 2 bars), while the pruned network ($\lambda = 13.4$) has less than 1\% of weight gradients falling into those regions. It suggests that the magnitude of weight gradients tends to decrease as the network is iteratively pruned.

{\bf (3) New Insight. }During the backpropagation, the weight update of $w_i$ is $ w_i \leftarrow w_i +  \alpha \frac{\partial \mathcal{L}}{\partial w_i}$, where $\alpha$ is the LR and $\mathcal{L}$ is the loss function. Assume that $\alpha$ is well-tuned to ensure the weight update (i.e.,  $\alpha \frac{\partial \mathcal{L}}{\partial w_i}$)  is sufficiently large to prevent the network from getting stuck in local optimal points \cite{bengio2012practical,goodfellow2016deep}. As shown in Fig. \ref{mlp_no_norm}(a), the magnitude of the weight gradient (i.e., $\frac{\partial \mathcal{L}}{\partial w_i}$) tends to decrease as the network is iteratively pruned. To preserve the same weight updating size and effect as before, a gradually larger value of LR ($\alpha$) should be used to retrain the pruned network during iterative pruning.

{\bf (4) Result Analysis. }We now provide an explanation for the change in the distribution of weight gradients. We assume each $x_i w_i$ (i.e., $x_i \in \mathbb{R}$ is the neuron input and $w_i \in \mathbb{R}$ is the associated weight) is an i.i.d. random variable. Then, the variance of the neuron's pre-activation output ($ \sum_{i=1}^{n} x_i w_i$, $n$ is the number of inputs) will be $\sum_{i=1}^{n} \text{Var} (x_i w_i)$. Pruning the network is equivalent to reducing the number of inputs from $n$ to $n - k$. This results in a smaller variance of $\sum_{i=1}^{n-k} \text{Var} (x_i w_i)$, leading to a smaller standard deviation. Hence, the distribution of the pre-activation output after pruning is narrower. Since ReLU returns its raw input if the input is non-negative, the distribution of hidden representations (output of hidden layers) becomes narrower as well. This can be verified from Fig. \ref{mlp_no_norm}(b), where we plot the distribution of hidden representations from the previous experiment. The key is that the weight gradient $\frac{\partial \mathcal{L}}{\partial w_i} $ is proportional to the hidden representation $x_i$ that associates with $w_i$ (i.e., $\frac{\partial \mathcal{L}}{\partial w_i} \propto x_i$). As the network is iteratively pruned, the distribution of hidden representations becomes narrower, leading to a narrower distribution of weight gradients. As a result, a larger LR should be used to retrain the pruned network.

{\bf (5) More Generalized Results. } (i) Effect of Batch Normalization (BN) \cite{ioffe2015batch}: BN is a popular technique to reformat the distribution of hidden representations, so as to address the issue of internal covariate shift. We note that similar performance trends can be observed after applying BN as well  (see Fig. \ref{mlp_distribution_norm} in the {\bf Appendix}). (ii) Popular CNN Networks and Pruning Methods: In addition to the global magnitude used before, two unstructured pruning methods (i.e., layer magnitude, global gradient) suggested by \cite{Blalock2020} and one structured pruning method (\texttt{L1} norm pruning) \cite{li2016pruning} are examined as well. Those methods are used to iteratively prune AlexNet \cite{krizhevsky2010convolutional}, ResNet-20 and VGG-19 using CIFAR-10. The results using these popular neural networks largely mirror those in Figs. \ref{mlp_no_norm}(a) \& (b) as well. We refer the interested reader to Figs. \ref{alex} - \ref{vgg} in the {\bf Appendix}.

\subsection{Theoretical Study and Motivation}
\label{sec3.2}
In this subsection, we theoretically investigate how network pruning can influence the value of the desired LR. The proofs of the results given here are provided in the {\bf Appendix}. First, we present some definitions. 

\begin{definition}
	\sloppy \textbf{Average Activation Energy ($E_{AA}$):} Given a network with fixed weights, input $X$ from a distribution $P$, and a layer $H=\{h_1(X),...,h_N(X)\}$ with $N$ nodes where $h_i(X)$ represents the function at the $i^{th}$ node. Then $E_{AA}(H)=\mathop{\mathbb{E}}_X[\frac{1}{N}\sum_i h_i(X)^2]$. This quantity reflects the average strength of the layer's activations.
\end{definition}
\begin{definition}
	\textbf{Weight-Gradient Energy ($E_{WG}$):} Let $W=[w_1,...,w_k]$ and $W'=[w'_1,...,w'_k]$ represent the flattened weight vector before and after one epoch of training. Then $E_{WG}$ is the average change in weight magnitude before and after a single epoch of training (for the active unpruned weights), i.e., $E_{WG}(W,W')=\mathbb{E}_{i}[(w_i-w'_i)^2]$. This measure quantifies how much the weights change after a single epoch of training.
\end{definition}
We now demonstrate the impact of network pruning on the \textit{average activation energy} of hidden layers. 

\begin{theorem}\label{thm:1}
	Consider a ReLU activated neural network represented as $X\xrightarrow[]{W_1}H\xrightarrow[]{W_2}Y$, where $X\in \mathbb{R}^d$ is the input, $H=\{H_1(X),H_2(X),...,H_N(X)\}$ is of infinite width ($N=\infty$), and $Y$ is the network output. $W_1$ and $W_2$ represent network parameters (weights, biases). Furthermore, let $X\sim \mathcal{N}(0,\sigma_{\scaleto{X}{3pt}}^2 I)$ and $W_1\sim\mathcal{N}(0,\sigma_{\scaleto{W}{3pt}}^2 I)$, where $I$ is the identity matrix and $\sigma_{\scaleto{X}{3pt}},\sigma_{\scaleto{W}{3pt}}$ are scalars. Now, let us consider an iterative pruning method, where in each iteration a fraction $0\leq p\leq1$  of the smallest magnitude weights are pruned (layer-wise pruning). Then, after $k$ iterations of pruning, it holds that
	\begin{dmath} \label{eq:thm1}
            \begin{split}
			&  \hspace*{-5mm} 4E_{AA}(H) \geq \sigma_W^2 + d\sigma_X^2\sigma_W^2 \Big((1-p)^k + \\ & \hspace*{2mm}\sqrt{\frac{4}{\pi}}\erfi\Big(1-(1-p)^k\Big)e^{-\left(\erfi\left(1-(1-p)^k\right)\right)^2}\Big)
            \end{split}
	\end{dmath}
	where $\erfi(.)$ is the inverse error function. 
\end{theorem}
Next, using the above result, the following theorem establishes how the \textit{weight-gradient energy} depends on the LR of the network, and the pruning iteration.
\begin{theorem} \label{thm:2}
	In Theorem \ref{thm:1}'s setting, we consider a single epoch of weight update for the network across a training dataset $S=\{(X_1,Y_1),..,(X_n,Y_n)\}$ using the cross-entropy loss, where $Y_i\in\{0,1\}$. Let $\alpha$ denote the learning rate. Let us denote the R.H.S of \eqref{eq:thm1} by $C(\sigma_X,\sigma_W,p,k)$. Let the final layer weights before and after one training epoch be $W_2$ and $W_2'$ respectively. We have, 
	\begin{equation}\label{eq:lr_energy}
		\resizebox{0.90\hsize}{!}{
		$\mathop{{}\mathbb{E}}_{W_2\sim\mathcal{N}_k(0,\sigma_{\scaleto{W}{3pt}}^2 I)}\left[E_{WG}\left(W_2,W_2'\right)\right] \geq \alpha^2\gamma C(\sigma_X,\sigma_W,p,k)$,}
	\end{equation}
	where $\mathcal{N}_k(0,\sigma_{\scaleto{W}{3pt}}^2 I)$ represents a normal distribution $\mathcal{N}(0,\sigma_{\scaleto{W}{3pt}}^2 I)$ initialization of $W_2$, followed by $k$ iterations of pruning, and $\gamma$ is a constant.
\end{theorem}
\begin{remark} {\bf (Pruning and LR)}
	Theorems \ref{thm:1} and \ref{thm:2} together establish how the choice of LR influences the lower bound of {\em weight-gradient energy}. Theorem \ref{thm:1} shows that the lower bound of {\em activation energy} of the hidden layer decreases as the network is pruned, and as Theorem \ref{thm:2} shows, this also reduces the lower bound of \textit{weight-gradient energy} per epoch. Thus, to counter this reduction, it is necessary to increase the learning rate $\alpha$ as the number of pruning cycles grows, in order to ensure that the R.H.S of \eqref{eq:lr_energy} remains fixed. 
\end{remark}
\begin{remark} {\bf (S-shape of LR During Iterative Pruning)}
	\sloppy Theorem \ref{thm:2} implies that to maintain a fixed \textit{weight-gradient energy} of $\mathop{{}\mathbb{E}}_{W_2\sim\mathcal{N}_k(0,\sigma_{\scaleto{W}{3pt}}^2 I)}[E_{WG}(W_2,W_2')] = K$, we must have the learning rate $\alpha \leq (K/\gamma C(\sigma_X,\sigma_W,p,k))^{1/2}$. We plot the adapted $\alpha$ assuming equality and find that this resembles an S-shape trajectory during iterative pruning.
\end{remark}

\section{A New Learning Rate Schedule}
\label{sec4}

In Section \ref{sec4.1}, we first review existing works on LR schedules and shortlist four benchmarks for comparison. Next, in Section \ref{sec4.2}, we introduce SILO and highlight the difference with existing works. Lastly, in Section \ref{sec4.3}, we detail the algorithm of the proposed SILO.

\begin{figure*}[!t]
	\centering
	\resizebox{0.88\linewidth}{!}{
		\begin{tikzpicture}[yscale=2.6, xscale=1.0]
			\draw[line width = 0.5mm](-5,-0.05)--(0.2,-0.05);
			\draw[line width = 0.3mm](0.1,-0.12)--(0.3,0.03);
			\draw[line width = 0.3mm](0.4,-0.12)--(0.6,0.03);
			
			\draw[line width = 0.5mm](0.5,-0.05)--(5.5,-0.05);
			\draw[line width = 0.3mm](5.4,-0.12)--(5.6,0.03);
			\draw[line width = 0.3mm](5.7,-0.12)--(5.9,0.03);
			
			\draw[line width = 0.5mm](5.8,-0.05)--(10.7,-0.05);
			\draw[line width = 0.3mm](10.6,-0.12)--(10.8,0.03);
			\draw[line width = 0.3mm](10.9,-0.12)--(11.1,0.03);
			
			\draw[line width = 0.5mm](11,-0.05)--(16.1,-0.05);
			
			\draw[line width = 0.5mm,->](16.1,-0.05)--(18,-0.05);
			\draw[line width = 0.5mm,->](-5,-0.05)--(-5,2.2)node[right=0.4cm,below=0.1cm]{};
			\draw[line width = 0.5mm,->](-5,2.15)--(18,2.15)node[right=-3cm, below=0.1cm]{};
			\node[rotate=90,above=0.1cm]at (-5,0.8){\large \textbf{\normalsize Learning Rate}};
			\draw[domain=1.5:15, samples=100,thick, red,line width=2pt] plot (\x,{(1/(1+exp((-\x+8)))+0.7});
			\node[above=0.1cm]at (6.5,2.1){\normalsize \textbf{$\lambda$ = 100 $\leftarrow$ Percent of Weights Remaining $\rightarrow$ $\lambda$ = 0}};
			
			\begin{scope}[on background layer]
				\fill[top color=gray!90!green!20!white,bottom color=white] (-4.3, 1.0) rectangle (-0.15, -0.05);
				\fill[top color=green!80!red!20!white,bottom color=white] (1.0, 1.2) rectangle (5.15, -0.05);
				\fill[top color=green!60!blue!40!white,bottom color=white] (6.2, 1.7) rectangle (10.35, -0.05);
				\fill[top color=green!60!black!40!white,bottom color=white] (11.45, 2.0) rectangle (15.65, -0.05);
			\end{scope}

			\draw[red,line width=2pt](-2.9,0.7)--(1.5,0.7);
			\draw[blue, line width = 0.6mm](-4.3,-0.05)--(-2.8,0.71)
			--(-2.3,0.71)
			--(-2.3,0.071)
			--(-1.2,0.071)
			--(-1.2,0.0072)
			--(-0.2,0.0072)
			--(-0.2,0)
			--(-0.2,-0.04)
			;

			\draw[blue, line width = 0.6mm](1.0,-0.05)--(2.5,{(1/(1+exp((6.5)))+0.7})
			--(3,{(1/(1+exp((6.5)))+0.7})
			--(3,{((1/10)*(1/(1+exp((6.5)))+0.7})
			--(4.1,{((1/10)*(1/(1+exp((6.5)))+0.7})
			--(4.1,{((1/100)*(1/(1+exp((6.5)))+0.7})
			--(5.1,{((1/100)*(1/(1+exp((6.5)))+0.7})
			--(5.1,{((1/1000)*(1/(1+exp((6.5)))+0.7})
			--(5.1,-0.04)
			;
			
			\draw[blue, line width = 0.6mm](6.2,-0.05)--(7.7,{(1/(1+exp((1.5)))+0.92})
			--(8.2,{(1/(1+exp((1.5)))+0.92})
			--(8.2,{((1/10)*(1/(1+exp((1.5)))+0.92})
			--(9.3,{((1/10)*(1/(1+exp((1.5)))+0.92})
			--(9.3,{((1/100)*(1/(1+exp((1.5)))+0.92})
			--(10.3,{((1/100)*(1/(1+exp((1.5)))+0.92})
			--(10.3,{((1/1000)*(1/(1+exp((1.5)))+0.92})
			--(10.3,-0.04)
			;
			
			\draw[blue, line width = 0.6mm](11.5,-0.05)--(13,{(1/(1+exp((-5)))+0.7})
			--(13.5,{(1/(1+exp((-5)))+0.7})
			--(13.5,{((1/10)*(1/(1+exp((-5)))+0.7})
			--(14.6,{((1/10)*(1/(1+exp((-5)))+0.7})
			--(14.6,{((1/100)*(1/(1+exp((-5)))+0.7})
			--(15.6,{((1/100)*(1/(1+exp((-5)))+0.7})
			--(15.6,{((1/1000)*(1/(1+exp((-5)))+0.7})
			--(15.6,-0.04)
			;
			
			%
			\draw[line width = 0.5mm,<->](-4.3,-0.13)--(-0.2,-0.13);
			\draw[line width = 0.5mm,<->](1.0,-0.13)--(5.1,-0.13);
			
			\draw[line width = 0.5mm,<->](6.2,-0.13)--(16.6,-0.13);
			
			\foreach \P/\N in {(-2.2,-0.15)/\textbf{{\normalsize First~Pruning~Cycle}},(2.9,-0.15)/\textbf{\hspace*{3mm}{\normalsize Pruning~Cycle}}~$\texttt{q}$,(10.7,-0.15)/\textbf{\hspace*{5mm}{\normalsize Subsequent~Pruning~Cycles}}}{
				\node[below] at \P {\Large \N};
			}
			
			\foreach \M in {(-2.8, 1.75), (2.5,1.75),(7.7,2.78),(13,4.2)}{
				\filldraw[yscale=0.4] \M circle (2pt);
			}
			
			\foreach \M in {(0.65,1.0),(0.4,1.0),(0.15,1.0), (5.4,1.0),(5.65,1.0),(5.9,1.0), (11.1,1.0),(10.85,1.0),(10.6,1.0) , (16.1,1.0),(16.35,1.0),(16.6,1.0)}{
				\filldraw[yscale=0.35] \M circle (1.5pt);
			}
			
			\draw[line width = 0.5mm, dashed, opacity=0.4](-2.9,0.69)--(16,0.69);
			\draw[line width = 0.5mm, dashed, opacity=0.4](14.3,1.7)--(16.0,1.7);
			
			\node[above=0.1cm]at (16.8, 1.55){\Large  $\delta$ + $\epsilon$};
			\node[above=0.1cm]at (16.3, 0.57){\Large  $\epsilon$};

			\node[above=0.1cm]at (-1.7, 0.45){\large $\texttt{max\_lr}$};
			\node[above=0.1cm]at (3.6, 0.45){\large $\texttt{max\_lr}$};
			\node[above=0.1cm]at (8.7, 0.85){\large $\texttt{max\_lr}$};
			\node[above=0.1cm]at (14, 1.4){\large $\texttt{max\_lr}$};
			
			\draw[->,>=latex,thick] (-2.8,0.8) -- (-1.0,0.8); \node[above,rotate=0] at (-1.6,0.83){\textbf{No Growth}};
			
			\draw[->,>=latex,thick] (2,0.8) -- (3.5,0.85); \node[above,rotate=7] at (2.9,0.85){\textbf{Slow Growth \rom{1}}};
			\draw[->,>=latex,thick] (7,1.1) -- (8.5,1.5); \node[above,rotate=38]at (7.7,1.3){\textbf{Fast Growth}};
			\draw[->,>=latex,thick] (12,1.76) -- (14.5,1.8); \node[above,rotate=3]at (13.6,1.82){\textbf{Slow Growth \rom{2}}};
			
	\end{tikzpicture}}
	\vspace{-3mm}
	\caption{Illustration of SILO during pruning. The S-shape red line is motivated from Theorem \ref{thm:1}.}
	\label{SC}
	\vspace{-5mm}
\end{figure*}
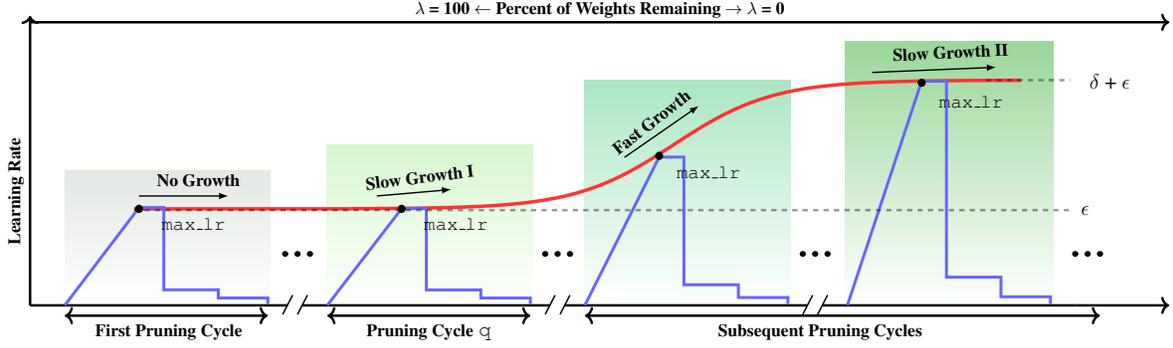

\subsection{LR Schedule Benchmarks}
\label{sec4.1}

Learning rate is the most important hyperparameter in training neural networks \cite{goodfellow2016deep}. The LR schedule is to adjust the value of LR during training by a pre-defined schedule. Three common LR schedules are summarized as follows.
\begin{enumerate}[itemsep = 2pt,leftmargin=5mm, topsep=1pt]
	
	\item {\bf LR Decay} starts with a large initial LR and linearly decays it by a certain factor after a pre-defined number of epochs. Several recent works \cite{you2019does,ge2019step,an2017exponential} have demonstrated that decaying LR helps the neural network to converge better and avoids undesired oscillations in optimization.
	
	\item {\bf LR Warmup} is to increase the LR to a large value over certain epochs and then decreases the LR by a certain factor. It is a popular schedule used by many practitioners for transfer learning \cite{he2019bag} and network pruning \cite{frankle2018lottery,frankle2020linear}.
	
	\item {\bf Cyclical LR} \cite{smith2017cyclical} varies the LR cyclically between a pre-defined lower and upper bound. It has been widely used in many tasks \cite{you2019does}.
	
\end{enumerate}

{\bf All the three LR schedules and constant LR} will be used as {\bf benchmarks for comparison}. 
We note that, in addition to LR schedules which vary LR by a pre-defined schedule, adaptive LR optimizers such as AdaDelta \cite{zeiler2012adadelta} and Adam \cite{kingma2014adam} provide heuristic based approaches to adaptively vary the step size of weight update based on observed statistics of the past gradients. All of them are sophisticated optimization algorithms and much work \cite{gandikota2021vqsgd,jentzen2021strong} has been done to investigate their mechanisms. In this paper, the performance of all benchmarks and SILO will be evaluated using SGD with momentum = 0.9 and a weight decay of 1$\texttt{e}$-4 (same as \cite{renda2020comparing,frankle2018lottery}). 
The effect of those adaptive LR optimizers on SILO will be discussed in Section \ref{A2} of the {\bf Appendix}.

\subsection{SILO Learning Rate Schedule}
\label{sec4.2}

To ensure the pruned network is properly trained during iterative pruning, we propose an S-shaped Improved Learning rate Schedule, called SILO, for iterative pruning of networks. As illustrated in Fig. \ref{SC}, the main idea of the proposed SILO is to apply the LR warmup schedule at every pruning cycle, with a gradual increase of the LR upper bound (i.e., \texttt{max\_lr}) in an S-shape as the network is iteratively pruned.  This LR warmup schedule is meant to be flexible and can change depending on different networks and datasets.

The S-shape in SILO is inspired by Theorem \ref{thm:2} (see Remark 2) and will be further verified by comparing to an Oracle. We divide the S-shape into four phases and provide the intuition behind each phase as follows.

\begin{enumerate}[noitemsep,leftmargin=5mm, topsep=0pt]
	
	\item Phase-1-{\bf No Growth}, SILO does not increase \texttt{max\_{lr}} until the pruning cycle \texttt{q} (see Fig. \ref{SC}). It is because the unpruned network often contains a certain amount of weights with zero magnitude. Those parameters are likely to be pruned at the first few pruning cycles, and removing such weights has negligible effect on the distribution of weight gradients. 
	
	\item Phase-2-{\bf Slow Growth \rom{1}}, the pruning algorithm has removed most zero magnitude weights and started pruning weights with small magnitude. Pruning such weights has a small effect on distribution of weight gradients. Hence, we slightly increase \texttt{max\_{lr}} after pruning cycle \texttt{q}.  
	
	\item Phase-3-{\bf Fast Growth}, SILO greatly increases \texttt{max\_{lr}}. It is because the pruning algorithm now starts removing weights with large magnitude and the distribution of weight gradients becomes much narrower. This requires a much larger LR for meaningful weight updates. 
	
	\item Phase-4-{\bf Slow Growth \rom{2}}, the network is now heavily pruned and very few parameters left in the network. By using the same pruning rate, a very small portion of the weights will be pruned. This could cause a marginal effect on the distribution of weight gradients. Hence, SILO slightly increases \texttt{max\_{lr}}.
	
\end{enumerate}

\noindent We note that SILO is designed based on the assumption that existing pruning methods tend to prune weights with small magnitude. 
{\bf The key difference with existing LR schedules} (e.g., Cyclical LR, LR warmup) is that SILO is adaptive and able to {\bf precisely increase} the value of \texttt{max\_{lr}} as the network is iteratively pruned, while existing LR schedules do not factor in the need to change \texttt{max\_{lr}} during different pruning cycles.

\begin{algorithm}[H] 
	\caption{Algorithm of the Proposed SILO}
	\label{algorithm1}
	\begin{algorithmic}[1]
		\Require{lower bound $\epsilon$, upper bound $\delta + \epsilon$, pruning rate \texttt{p}, number of pruning cycles $\texttt{L}$, number of training epochs $\texttt{t}$, S-shape control term $\beta$, delay term $\texttt{q}$.} 
		\For{$ \texttt{m} = 0$ to $\texttt{L}$}                    
		\If{$\texttt{m} \leq \texttt{q}$}
		\State $\texttt{max\_lr} = \epsilon $
		\Else
		\State  $\texttt{max\_lr} = \frac{\delta}{1 + (\frac{\gamma}{1-\gamma})^{-\beta}} + \epsilon$, $\gamma = 1-(1-\texttt{p})^{m - \texttt{q}}$
		
		\EndIf
		\EndFor
		
		\For {$ \texttt{i} = 0$ to $\texttt{t}$} 
		\State (1) linearly warmup the LR to $\texttt{max\_lr}$
		\State (2) drop the value of LR by 10 at certain epochs
		\EndFor
	\end{algorithmic}
\end{algorithm}

\subsection{Implementation of SILO}
\label{sec4.3}

As for the implementation of SILO, we designed a function to estimate the value of $\texttt{max\_lr}$ as shown below.
\vspace{-1mm}
\begin{equation}
	\label{SILO-eq}
	\texttt{max\_lr} = \frac{\delta}{1 + (\frac{\gamma}{1-\gamma})^{-\beta}} + \epsilon, 
\end{equation}
\vspace{-1mm}
where $\gamma = 1 - (1 - \texttt{p})^{m - \texttt{q}}$ is the input of the function and $\texttt{max\_lr}$ is the output of the function. The parameter \texttt{p} is the pruning rate and $\texttt{m}$ is the number of completed pruning cycles. The parameters $\beta$ and $\texttt{q}$ are used to control the shape of the S curve. The larger the $\beta$, the later the curve enters the Fast Growth phase. The parameter \texttt{q} determines at which pruning cycle SILO enters the Slow Growth \rom{1} phase. When $q = 0$, the No Growth phase will be skipped and $\gamma$ will be the proportion of pruned weights at the current pruning cycle. The parameters $\epsilon$ and $\delta$ determine the value range of $\texttt{max\_lr}$.
As the network is iteratively pruned, $\gamma$ increases and $\texttt{max\_lr}$ increases from $\epsilon$ to $\epsilon + \delta$ accordingly. The details of the SILO algorithm are summarized in Algorithm \ref{algorithm1}.

{\bf Parameter Selection for SILO.} Algorithm \ref{algorithm1} requires several inputs for implementation. The value of $\epsilon$ can be tuned using the validation accuracy of the unpruned network while the value of $\delta$ can be tuned using the validation accuracy of the pruned network with targeted sparsity. The pruning rate \texttt{p} and pruning cycles \texttt{L} are chosen to meet the target sparsity. The number of training epochs \texttt{t} should be large enough to guarantee the network convergence. Let $\texttt{q}$ = 1 and $\beta$ = 5 could be a good choice and 
yield promising results as we demonstrate in the Section of Performance Evaluation. Furthermore, based on our experience, the value of $\texttt{q}$ and $\beta$ could be tuned in the range of [0, 3] and [3, 6], respectively.

{\renewcommand{\arraystretch}{0.85}
	\begin{table}[!t]
		\centering
		\setlength\tabcolsep{5.5pt}
		\begin{tabular}{lcccc}
			\toprule
			\toprule
			\multicolumn{5}{c}{Original Top-1 Test Accuracy = 91.7\% ($\lambda$ = 100)}\\ \toprule
			\multicolumn{1}{c}{$\lambda$} & 32.8 & 26.2 & 8.59 & 5.72 \\ \toprule
			{\small constant LR } & 88.1$\pm${\scriptsize 0.9} & 87.5$\pm${\scriptsize 0.7}  & 82.8$\pm${\scriptsize 0.9} & 79.1$\pm${\scriptsize 0.8}  \\
			
			{\small LR decay} & 89.8$\pm${\scriptsize 0.4} & 89.0$\pm${\scriptsize 0.7} & 83.9$\pm${\scriptsize 0.6} & 79.8$\pm${\scriptsize 0.7}  \\
			
			{\small cyclical LR} & 89.7$\pm${\scriptsize 0.6} & 88.2$\pm${\scriptsize 0.7} & 84.1$\pm${\scriptsize 0.8} & 80.3$\pm${\scriptsize 0.7} \\
			
			{\small LR-warmup} & 90.3$\pm${\scriptsize 0.4} & 89.8$\pm${\scriptsize 0.6} & 85.9$\pm${\scriptsize 0.9} & 81.2$\pm${\scriptsize 1.1} \\
			
			{\small SILO (Ours)} & {\bf 90.8$\pm${\scriptsize 0.5}} & {\bf 90.3$\pm${\scriptsize 0.4}} & {\bf 87.5$\pm${\scriptsize 0.8}} & {\bf 82.7$\pm${\scriptsize 1.2}} \\
			\bottomrule
			\bottomrule
		\end{tabular}
		\vspace{-3mm}
		\caption{Top-1 test accuracy ($\pm$ std) of pruning ResNet-20 on CIFAR-10 using global magnitude.}
		\label{per1}
		\vspace{-1mm}
\end{table}}

{\renewcommand{\arraystretch}{0.85}
	\begin{table}[!t]
		\centering
		\setlength\tabcolsep{5.5pt}
		\begin{tabular}{lcccc}
			\toprule
			\toprule
			\multicolumn{5}{c}{Original Top-1 Test Accuracy = 92.2\% ($\lambda$ = 100)} \\ \toprule
			\multicolumn{1}{c}{$\lambda$} & 32.8 & 26.2 & 8.59 & 5.72 \\ \toprule
			{\small constant LR} & 88.8$\pm${\scriptsize 0.6}& 87.4$\pm${\scriptsize 0.7} & 82.2$\pm${\scriptsize 1.4} & 73.7$\pm${\scriptsize 1.3}  \\
			
			{\small LR decay} & 89.4$\pm${\scriptsize 0.4} & 88.6$\pm${\scriptsize 0.5}  & 83.3$\pm${\scriptsize 0.8} & 75.4$\pm${\scriptsize 0.9}  \\
			
			{\small cyclical LR} & 89.8$\pm${\scriptsize 0.5} & 89.1$\pm${\scriptsize 0.6} & 83.7$\pm${\scriptsize 1.0} & 75.7$\pm${\scriptsize 1.2} \\
			
			{\small LR-warmup} & 90.2$\pm${\scriptsize 0.5} & 89.8$\pm${\scriptsize 0.8} & 84.5$\pm${\scriptsize 0.9} & 76.5$\pm${\scriptsize 1.0} \\
			
			{\small SILO (Ours)} & {\bf 90.6$\pm${\scriptsize 0.6}} & {\bf 90.3$\pm${\scriptsize 0.6}} &{\bf 86.1$\pm${\scriptsize 0.8}} & {\bf 78.5$\pm${\scriptsize 1.0}} \\
			\bottomrule
			\bottomrule
		\end{tabular}
		\vspace{-3mm}
		\caption{Top-1 test accuracy ($\pm$ std) of pruning VGG-19 on CIFAR-10 using global gradient.}
		\label{per2}
		\vspace{-2mm}
\end{table}}

\section{Performance Evaluation}
\label{PE}

\subsection{Experimental Setup}
\label{ES}
We demonstrate that SILO can work well with different pruning methods across a wide range of networks and datasets. The details for each experiment are as follows.

\begin{enumerate}[itemsep = 1pt,leftmargin=5mm, topsep=0pt]
	\item Pruning ResNet-20 \cite{resnet18} on CIFAR-10 via global magnitude (i.e., prune weights with the lowest magnitude anywhere in the networks).
	
	\item Pruning VGG-19 \cite{vgg16} on CIFAR-10 via global gradient (i.e., prune weights with the lowest magnitude of (weight $\times$ gradient) anywhere in the network).
	
	\item Pruning DenseNet-40 \cite{huang2017densely} on the CIFAR-100 dataset using Layer-adaptive Magnitude-based Pruning (LAMP) \cite{lee2020layer}.
	
	\item Pruning MobileNetV2 \cite{sandler2018mobilenetv2} on the CIFAR-100 dataset using Lookahead Pruning (LAP) \cite{park2020lookahead}.
	
	\item Pruning ResNet-50 on ImageNet (i.e., ImageNet-1000) \cite{deng2009imagenet} using Iterative Magnitude Pruning (IMP) \cite{frankle2018lottery}. 
	\item Pruning Vision Transformer (ViT-B-16) \cite{dosovitskiy2020image} on CIFAR-10 using IMP.
\end{enumerate}
In each experiment, we compare SILO ($\texttt{q}$ = 1, $\beta$ =5) to constant LR and the three shortlisted LR schedules: (i) LR decay, (ii) cyclical LR and (iii) LR warmup. The details of each LR schedule are  summarized in Table \ref{schedules} in the {\bf Appendix}.

\textbf{(1) Methodology.} We train the network using the training dataset via SGD with momentum = 0.9 and a weight decay of 1$\texttt{e}$-4 (same as \cite{renda2020comparing,frankle2018lottery}). Next, we prune the trained network with a pruning rate of 0.2 (i.e., 20\% of remaining weights are pruned) in 1 pruning cycle. We repeat 25 pruning cycles in 1 run and use early-stop top-1 test accuracy (i.e., the corresponding test accuracy when early stopping criteria for validation error is met) to evaluate the performance. The results are averaged over 5 runs and the corresponding standard deviation are summarized in Tables \ref{per1} - \ref{per5}, where the results of pruning ResNet-20, VGG-19, DenseNet-40, MobileNetV2, ResNet-50 and Vision Transformer (ViT-B-16) are shown, respectively. Some additional details (e.g., training epochs) and  results are given in Tables \ref{per1_extra} - \ref{per5_extra} in the {\bf Appendix}.



\textbf{(2) Parameters for SOTA LR Schedules.} To ensure fair comparison against prior SOTA LR schedules, we utilize implementations reported in the literature. Specifically, the LR schedules (i.e., LR-warmup) from Table \ref{per1} - \ref{per4} are from \cite{frankle2018lottery}, \cite{frankle2020linear}, \cite{zhao2019variational}, \cite{chin2020towards} and \cite{renda2020comparing}, respectively.
The LR schedule (i.e., cosine decay) in Table \ref{per5} is from \cite{dosovitskiy2020image}.



\textbf{(3) Parameters for other LR schedules.} For the other schedules without a single "best" LR in the literature, we tune the value of LR for each of them via a grid search with a range from 1\texttt{e}-4 to 1\texttt{e}-1 using the validation accuracy. Other related parameters (e.g., step size) are also tuned in the same manner. Lastly, we highlight that all LR schedules used, including SILO, are rewound to the initial state at the beginning of each pruning cycle, which is the same as the LR rewinding in \cite{renda2020comparing}. 

\textbf{(4) Source Code \& Devices:} We use Tesla V100 devices for our experiments, and the source code (including random seeds) will be released at the camera-ready stage.

{\renewcommand{\arraystretch}{0.85}
	\begin{table}[!t]
		\centering
		\setlength\tabcolsep{5.5pt}
		\begin{tabular}{lcccc}
			\toprule
			\toprule
			\multicolumn{5}{c}{Original Top-1 Test Accuracy = 74.6\% ($\lambda$ = 100)} \\ \toprule
			\multicolumn{1}{c}{$\lambda$} & 32.8 & 26.2 & 8.59 & 5.72 \\ \toprule
			
			{\small constant LR} & 70.3$\pm${\scriptsize 0.8}& 68.1$\pm${\scriptsize 0.7} & 60.8$\pm${\scriptsize 1.1} & 59.1$\pm${\scriptsize 1.2}  \\
			
			{\small LR decay} & 71.2$\pm${\scriptsize 0.8} & 69.0$\pm${\scriptsize 0.6} & 62.6$\pm${\scriptsize 1.2} & 60.3$\pm${\scriptsize 1.4}  \\
			
			{\small cyclical LR} & 70.9$\pm${\scriptsize 0.6} & 69.4$\pm${\scriptsize 0.6} & 63.0$\pm${\scriptsize 1.1} & 60.8$\pm${\scriptsize 1.3} \\
			
			{\small LR-warmup} & 71.5$\pm${\scriptsize 0.7} & 69.6$\pm${\scriptsize 0.8} & 63.9$\pm${\scriptsize 1.0} & 61.2$\pm${\scriptsize 0.9} \\
			
			{\small SILO (Ours)} & {\bf 72.4$\pm${\scriptsize 0.7}} & {\bf 70.8$\pm${\scriptsize 0.8}} &{\bf 65.7$\pm${\scriptsize 1.2}} & {\bf 63.7$\pm${\scriptsize 1.0}} \\
			
			\bottomrule
			\bottomrule
		\end{tabular}
		\vspace{-3mm}
		\caption{Top-1 test accuracy ($\pm$ std) of pruning DenseNet-40 on CIFAR-100 using LAMP \cite{lee2020layer}. }
		\label{per3}
		\vspace{-1mm}
\end{table}}

{\renewcommand{\arraystretch}{0.85}
	\begin{table}[!t]
		\centering
		\setlength\tabcolsep{5.5pt}
		\begin{tabular}{lcccc}
			\toprule
			\toprule
			\multicolumn{5}{c}{Original Top-1 Test Accuracy = 73.7\% ($\lambda$ = 100)} \\ \toprule
			\multicolumn{1}{c}{$\lambda$} & 32.8 & 26.2 & 8.59 & 5.72 \\ \toprule
			
			{\small constant LR} & 69.8$\pm${\scriptsize 1.1}& 68.2$\pm${\scriptsize 0.9} & 63.8$\pm${\scriptsize 1.1} & 62.1$\pm${\scriptsize 1.2}  \\
			
			{\small LR decay} & 70.9$\pm${\scriptsize 1.0} & 69.4$\pm${\scriptsize 0.6} & 65.1$\pm${\scriptsize 0.8} & 64.0$\pm${\scriptsize 1.1}  \\
			
			{\small cyclical LR} & 71.5$\pm${\scriptsize 0.7} & 69.6$\pm${\scriptsize 0.6} & 65.3$\pm${\scriptsize 1.1} & 64.3$\pm${\scriptsize 1.2} \\
			
			{\small LR-warmup} & 72.1$\pm${\scriptsize 0.8} & 70.5$\pm${\scriptsize 0.9} & 66.2$\pm${\scriptsize 1.1} & 64.8$\pm${\scriptsize 1.5} \\
			
			{\small SILO (Ours)} & {\bf 72.5$\pm${\scriptsize 0.6}} & {\bf 71.0$\pm${\scriptsize 0.7}} &{\bf 68.8$\pm${\scriptsize 0.8}} & {\bf 66.8$\pm${\scriptsize 1.4}} \\
			
			\bottomrule
			\bottomrule
		\end{tabular}
		\vspace{-3mm}
		\caption{Top-1 test accuracy ($\pm$ std) of pruning MobileNetV2 on CIFAR-100 using LAP \cite{lee2020layer}. }
		\label{per3_add}
		\vspace{-2mm}
\end{table}}

{\renewcommand{\arraystretch}{0.85}
	\begin{table}[!t]
		\centering
		\setlength\tabcolsep{5.5pt}
		\begin{tabular}{lcccc}
			\toprule
			\toprule
			\multicolumn{5}{c}{Original Top-1 Test Accuracy = 77.0\% ($\lambda$ = 100)} \\ \toprule
			\multicolumn{1}{c}{$\lambda$} & 32.8 & 26.2 & 8.59 & 5.72 \\ \toprule
			
			{\small constant LR} & 74.2$\pm${\scriptsize 0.8} & 73.9$\pm${\scriptsize 0.7} & 70.5$\pm${\scriptsize 0.6} & 69.2$\pm${\scriptsize 0.9} \\
			
			{\small LR decay} & 75.6$\pm${\scriptsize 0.5} & 75.1$\pm${\scriptsize 0.5}  & 72.7$\pm${\scriptsize 0.8} & 70.5$\pm${\scriptsize 0.6}  \\
			
			{\small cyclical LR} & 76.5$\pm${\scriptsize 0.5} & 75.5$\pm${\scriptsize 0.6} & 73.4$\pm${\scriptsize 0.8} & 71.2$\pm${\scriptsize 0.7} \\
			
			{\small LR-warmup} & 76.6$\pm${\scriptsize 0.2} & 75.8$\pm${\scriptsize 0.3} & 73.8$\pm${\scriptsize 0.5} & 71.5$\pm${\scriptsize 0.4} \\
			
			{\small SILO (Ours)} & {\bf 76.8$\pm${\scriptsize 0.4}} & {\bf 76.1$\pm${\scriptsize 0.7}} & {\bf 75.2$\pm${\scriptsize 0.8}} & {\bf 73.8$\pm${\scriptsize 0.6}} \\
			\bottomrule
			\bottomrule
		\end{tabular}
		\vspace{-3mm}
		\caption{Top-1 test accuracy ($\pm$ std) of pruning ResNet-50 on ImageNet using IMP \cite{frankle2018lottery}.}
		\label{per4}
		\vspace{-1mm}
\end{table}}

{\renewcommand{\arraystretch}{0.85}
	\begin{table}[!t]
		\centering
		\setlength\tabcolsep{5.5pt}
		\begin{tabular}{lcccc}
			\toprule
			\toprule
			\multicolumn{5}{c}{Original Top-1 Test Accuracy = 98.0\% ($\lambda$ = 100)} \\ \toprule
			
			\multicolumn{1}{c}{$\lambda$} & 32.8 & 26.2  & 8.59 & 5.72 \\ \toprule
			
			{\small constant LR} & 96.4$\pm${\scriptsize 0.5} & 96.0$\pm${\scriptsize 0.7} & 83.0$\pm${\scriptsize 0.9} & 80.1$\pm${\scriptsize 0.8} \\
			
			{\small cosine decay} & 97.2$\pm${\scriptsize 0.2} & 96.5$\pm${\scriptsize 0.6}  & 84.1$\pm${\scriptsize 1.0} & 81.6$\pm${\scriptsize 1.1}  \\
			
			{\small cyclical LR} & 97.0$\pm${\scriptsize 0.2} & 96.5$\pm${\scriptsize 0.6} & 83.4$\pm${\scriptsize 0.6} & 81.0$\pm${\scriptsize 1.1} \\
			
			{\small LR-warmup} & 97.3$\pm${\scriptsize 0.6} & 96.8$\pm${\scriptsize 0.7} & 84.4$\pm${\scriptsize 0.8} & 82.1$\pm${\scriptsize 0.9} \\
			
			{\small SILO (Ours)} & {\bf 97.7$\pm${\scriptsize 0.5}} & {\bf 97.4$\pm${\scriptsize 0.6}} & {\bf 85.5$\pm${\scriptsize 0.9}} & {\bf 83.4$\pm${\scriptsize 0.8}} \\
			\bottomrule
			\bottomrule
		\end{tabular}
		\vspace{-3mm}
		\caption{Top-1 test accuracy ($\pm$ std) of pruning Vision Transformer on CIFAR-10 using IMP \cite{frankle2018lottery}.}
		\label{per5}
		\vspace{-3mm}
\end{table}}


\subsection{Performance Comparison}
\label{PC}


\textbf{(1) Reproducing SOTA results.} By using the implementations reported in the literature, we have correctly reproduced SOTA results. For example, the benchmark results of LR warmup in our Tables \ref{per1} - \ref{per5} are comparable to Fig.11 and Fig.9 of \cite{Blalock2020}, Table.4 in \cite{liu2018rethinking}, Fig.3 in \cite{chin2020towards}, Fig. 10 in \cite{frankle2020linear}, Table 5 in \cite{dosovitskiy2020image}, respectively. 


\textbf{(2) SILO outperforms SOTA results.} The key innovation of SILO is that the LR precisely increases as the network is pruned, by increasing $\texttt{max\_lr}$ in an S-shape as $\lambda$ decreases. This results in a much higher accuracy than all LR schedule benchmarks studied. For example, in Table \ref{per1}, the top-1 test accuracy of SILO is 1.8\% higher than the best performing schedule (i.e., LR-warmup) at $\lambda$ = 5.72. SILO also obtains the best performance when using larger models in Table \ref{per2} (i.e., 2.6\% higher at $\lambda$ = 5.72) and using more difficult datasets in Table \ref{per3} (i.e., 4.0\% higher at $\lambda$ = 5.72).

\textbf{(3) Performance on ImageNet.} In Table \ref{per4}, we show the performance of SILO using IMP (i.e., the lottery ticket hypothesis pruning method) via ResNet-50 on ImageNet (i.e., the ILSVRC version) which contains over 1.2 million images from 1000 different classes. We observe that SILO still outperforms the best performing LR schedule benchmark (LR-warmup) by 1.9\% at $\lambda = 8.59$. This improvement increases to 3.2\% when $\lambda$ reduces to 5.72. 

\textbf{(4) Performance on SOTA networks (Vision Transformer).} Several recent works \cite{liu2021swin,yuan2021volo} demonstrated that transformer based networks tend to provide excellent performance in computer vision tasks (e.g., classification). We now examine the performance of SILO using Vision Transformer (i.e., ViT-B16 with a resolution of 384). We note that the ViT-B16 uses Gaussian Error Linear Units (GELU, GELU(x) = x$\Phi(x)$, where $\Phi(x)$ is the standard Gaussian cumulative distribution function) as the activation function. We note that both ReLU and GELU have the unbounded output, suggesting that SILO could be helpful for pruning GELU based models as well.

We repeat the same experiment setup as above and compare the performance of SILO to other LR schedules using ViT-B16 in Table \ref{per5}. We observe that SILO is able to outperform the standard implementation (cosine decay, i.e., decay the learning rate via the cosine function) by 1.3\% at $\lambda = 8.59$ in top-1 test accuracy. This improvement increases to 1.6\% when $\lambda$ reduces to 5.72.



{\renewcommand{\arraystretch}{1.0}
	\begin{table}[!t]
		\centering
		\setlength\tabcolsep{1.2pt}
		\begin{tabular}{l|c|c|c|c|c}
			\bottomrule
			\bottomrule
			\multicolumn{1}{c|}{\small $\lambda$} & 100 & 51.3 & 32.9 & 21.1 & 5.72\\ \midrule
			{\small  Oracle \texttt{max\_lr} } & 4 & 4.6 & 9.0 & 9.8 & 10.2 \\
			{\small  Oracle interval  } & {\footnotesize [3.6,4.2]} & {\footnotesize [4.2,5.4]} & {\footnotesize [8.0,9.6]} & {\footnotesize [9.2,10.4]} & {\footnotesize [9.8,10.6]}  \\
			{\small SILO \texttt{max\_lr}} & 4 & 4.32 & 9.2 & 9.9 & 9.99 \\
			\bottomrule
			\bottomrule
		\end{tabular}
		\vspace{-3mm}
		\caption{Comparison between Oracle $\texttt{max\_lr}$, Oracle interval (both obtained via grid search) and $\texttt{max\_lr}$ estimated by SILO when iteratively pruning VGG-19 on CIFAR-10. Note that all values in the table are in hundredths. }
	\vspace{-0mm}
	\label{lr_table_vgg}
	\vspace{-5mm}
\end{table}}

\subsection{Comparing SILO to an Oracle}
\label{5.3}

Our new insight suggests that, due to the change in distribution of hidden representations during iterative pruning, LR should be re-tuned at each pruning cycle. SILO provides a method to adjust the \texttt{max\_{lr}} in an S-shape, which is backed up by a theoretical result (see Theorem \ref{thm:2}). We now further examine the S-shape trajectory of SILO by comparing SILO's estimated \texttt{max\_{lr}} to an Oracle, which uses the same LR warmup structure as SILO but exhaustively searches for the optimal value of \texttt{max\_{lr}} at each pruning cycle. The Oracle's \texttt{max\_{lr}} at the current pruning cycle is chosen by grid search ranging from 1$\texttt{e}$-4 to 1$\texttt{e}$-1 and the best performing value (i.e., determined by validation accuracy) is used to train the network. The results of \texttt{max\_{lr}} determined this way when iteratively pruning a VGG-19 on CIFAR-10 using the global magnitude are detailed in Table \ref{lr_table_vgg} via two metrics:
\begin{enumerate}[noitemsep,leftmargin=5mm, topsep=0pt]
	\item {\bf Oracle $\texttt{max\_lr}$}: The value of $\texttt{max\_lr}$ that provides the best validation accuracy. 
	\item {\bf Oracle interval}: The value range of $\texttt{max\_lr}$ which performs within 0.5\% of the best validation accuracy.
\end{enumerate}

\textbf{SILO vs Oracle (Performance):} In Table \ref{lr_table_vgg}, we find that the value of $\texttt{max\_lr}$ estimated by SILO falls in the Oracle optimized 
$\texttt{max\_lr}$ interval at each pruning cycle. It means that SILO is able to precisely adjust $\texttt{max\_lr}$ to provide competitive performance with the Oracle. This further verifies the S-shape trajectory of $\texttt{max\_lr}$ used in the SILO.

\textbf{SILO vs Oracle (Complexity):} The process of finding the Oracle tuned $\texttt{max\_lr}$ requires a significantly larger computational complexity in tuning due to the grid search. Assume that $\texttt{max\_lr}$ is searched from a sampling space of [$\theta_1$, $\cdots$, $\theta_n$] for $k$ pruning cycles. Hence, the complexity of the Oracle will be $\mathcal{O}(n^k)$. On the other hand, SILO controls the variation of $\texttt{max\_lr}$ at each pruning cycle via four parameters: ranges of $\texttt{max\_lr}$: [$\epsilon$, $\epsilon + \delta$], delay term \texttt{q} and S-shape control term $\beta$. Similar to the Oracle, both $\epsilon$ and $\delta$ can be searched from a range of $n$ values. As we have recommended before, $\texttt{q}$ and $\beta$ can be tuned in the range of [0, 3], [3, 6], respectively. As a result, SILO has a complexity of $\mathcal{O}(n^2)$, which is exponentially less complex than the Oracle's complexity, but with competitive performance. Lastly, we highlight that similar performance trends can be observed using ResNet-20 (see Table \ref{lr_table_resnet_20} in the {\bf Appendix}).


\vspace{-2mm}
\section{Conclusion}

\label{discussion}
SILO is an adaptive LR schedule for network pruning with theoretical justification. SILO outperforms 
existing benchmarks by 2\% - 4\% via extensive experiments. Furthermore, via the S-shape trajectory, SILO obtains comparable performance to Oracle with significantly lower complexity.

\balance
\bibliography{anonymous-submission-latex-2023}

\clearpage
\newpage
\appendix
\onecolumn

\section{Proofs of Theoretical Results (Including new results)}
In this section, we provide the proofs of the theoretical results (Theorem 1 and Theorem 2) in the paper. Furthermore, we also provide two new theoretical results (Corollaries 1 and 2), which extend our results to a more general case of arbitrary network depth $D$. 

\setcounter{theorem}{0}
\subsection{Proof of Theorem 1, Theorem 2, Corollary 1 and Corollary 2}
\begin{theorem}
	Consider a single-hidden layer ReLU activated network represented as $X\xrightarrow[]{W_1}H\xrightarrow[]{W_2}Y$, where $X\in \mathbb{R}^d$ is the input, $H=\{H_1(X),H_2(X),...,H_N(X)\}$ is of infinite width ($N=\infty$), and $Y$ is the network output. $W_1$ and $W_2$ represent network parameters (weights, biases). Furthermore, let $X\sim \mathcal{N}(0,\sigma_{\scaleto{X}{3pt}}^2 I)$ and $W_1\sim\mathcal{N}(0,\sigma_{\scaleto{W}{3pt}}^2 I)$, where $I$ is the identity matrix and $\sigma_{\scaleto{X}{3pt}},\sigma_{\scaleto{W}{3pt}}$ are scalars. Now, let us consider an iterative pruning method, where in each iteration a fraction $0\leq p\leq1$  of the smallest magnitude weights are pruned (layer-wise pruning). Then, after $k$ iterations of pruning, it holds that
	\begin{equation} 
		\resizebox{0.9\hsize}{!}{
			$4E_{AA}(H)\geq\sigma_W^2 + d\sigma_X^2\sigma_W^2\left((1-p)^k +\sqrt{\frac{4}{\pi}}\erfi\left(1-(1-p)^k\right)e^{-\left(\erfi\left(1-(1-p)^k\right)\right)^2} \right)$}
	\end{equation}
	where $\erfi(.)$ is the inverse error function \cite{andrews1998special}. 
\end{theorem}
\begin{proof}
	Let us separate the weights and biases for the first layer into the matrix $W_1\in \mathbb{R}^{d\times N}$ and the bias vector $b \in \mathbb{R}^{N}$. We can thus compute the hidden layer activations as,
	\begin{equation}
		H_j(X)=max\left(0,b_j+\sum_{i=1}^dW_1^{ij}X_i\right)
	\end{equation}
	We can express $X=[X_1,X_2,..,X_d]$, where $X_i\sim \mathcal{N}(0,\sigma_X^2)$ are i.i.d random variables. Let us denote $Z_j(X)=b_j+\sum_{i=1}^dW_1^{ij}X_i$. We have that $\mathbb{E}_X[Z_j(X)]=\sum_{i=1}^d \mathbb{E}_X[W_1^{ij}X_i]+b_j=b_j\sum_{i=1}^d W_1^{ij}\mathbb{E}_X[X_i]=b_j$. We can similarly compute the variance of $Z_j(X)$, $\mathbb{E}_X[\left(Z_j(X)-b_j\right)^2]=\sigma_X^2\sum_{i=1}^d\left(W_1^{ij}\right)^2$, which follows from the fact that $X_1,X_2,...X_d$ are independent. Note that as the sum of independent Gaussian distributed variables is Gaussian, we have that $Z_j(X)\sim \mathcal{N}\left(b_j,\sigma_X^2\sum_{i=1}^d\left(W_1^{ij}\right)^2\right)$.
	
	\sloppy
	With this observation, we first show that for the random variable  $H_j(X)=max\left(0,b_j+\sum_{i=1}^dW_1^{ij}X_i\right)$, for $b_j\geq0$, $\mathbb{E}_X[H_j(X)^2]\geq \mathbb{E}_X[Z_j(X)^2]/2= \left( b_j^2 + \mathbb{E}_X\left[\left(\sum_{i=1}^dW_1^{ij}X_i\right)^2\right] \right)/2$. To show this, we first note that for all $x< 0$, $P(H_j(X)=x)=0$, $P(H_j(X)=0)= \int_{-\infty}^{0} P(Z_j(x))dx$, and for $x>0$, $P(H_j(X)=x)=P(Z_j(X)=x)$. Note that for $b_j\geq0$, $ \int_{-\infty}^{0} P(Z_j(x))dx\leq \frac{1}{2}$. Let us define $ c = 1-\int_{-\infty}^{0} P(Z_j(x))dx$. Thus, it holds that $c\geq\frac{1}{2}$. As $Z_j(X)\sim \mathcal{N}\left(b_j,\sigma_X^2\sum_{i=1}^d\left(W_1^{ij}\right)^2\right)$, we have that 
	
	\begin{align}
		\mathbb{E}_X[H_j(X)^2]=&\int_{0}^{\infty} x^2 P(H_j(X)=x)dx = c\int_{0^{+}}^{\infty} x^2 \frac{P(H_j(X)=x)}{c}dx \\ =& c\int_{0^{+}}^{\infty} x^2 \frac{P(Z_j(X)=x)}{c}dx =  c\int_{-\infty}^{\infty} x^2 {P_{trunc}(Z_j(X)=x)}dx, \label{eq:h_ener}
	\end{align}
	
	where ${P_{trunc}(Z_j(X))}$ is a truncated version of the normal distribution $\mathcal{N}\left(b_j,\sigma_X^2\sum_{i=1}^d\left(W_1^{ij}\right)^2\right)$, where all probability values for all $Z_j(X)\leq0$ are now zero. For simplicity of notation, we denote $\sigma_j^2=\sigma_X^2\sum_{i=1}^d\left(W_1^{ij}\right)^2$. Let $\mu_{trunc}=\int_{-\infty}^{\infty} x {P_{trunc}(Z_j(X)=x)}dX$. Also, let $\phi(x)=\frac{1}{\sqrt{2\pi}}e^{-\frac{x^2}{2}}$. In what follows we use the expression for the mean and variance of a truncated normal distribution, 
	\begin{align}
		\int_{-\infty}^{\infty} x^2 &{P_{trunc}(Z_j(X)=x)} dx = \int_{-\infty}^{\infty} (x-\mu_{trunc})^2 {P_{trunc}(Z_j(X)=x)}dx + (\mu_{trunc})^2 \nonumber \\
		& \quad \quad \quad=  \sigma_j^2\left(1-\frac{b_j\phi\left(\frac{-b_j}{\sigma_j}\right)}{c\sigma_j} - \frac{\phi\left(\frac{-b_j}{\sigma_j}\right)^2}{c^2}\right)+ \left(b_j+\frac{\sigma_j\phi\left(\frac{-b_j}{\sigma_j}\right)}{c}\right)^2 \nonumber\\
		& \quad \quad \quad= \sigma_j^2+b_j^2+ \frac{b_j\sigma_j\phi\left(\frac{-b_j}{\sigma_j}\right)}{c}\\
		& \quad \quad \quad\geq \sigma_j^2+b_j^2 = \mathbb{E}_X[Z_j(X)^2]
	\end{align}
	where the last step follows from the fact that $b_j\geq0$. Combining this result with \eqref{eq:h_ener}, we have that
	\begin{equation}
		\mathbb{E}_X[H_j(X)^2]= c\int_{-\infty}^{\infty} x^2 {P_{trunc}(Z_j(X)=x)}dx\geq c\mathbb{E}_X[Z_j(X)^2]\geq \frac{1}{2}\mathbb{E}_X[Z_j(X)^2],
	\end{equation}
	Next, for the case when $b_j< 0$, we simply have that $\mathbb{E}_X[H_j(X)^2]\geq0$. As $b_j$ is distributed as $\mathcal{N}(0,\sigma_W)$ as well, we have ($P(b_j\leq0)=P(b_j>0)=0.5$). As $N=\infty$, we can write, 
	\begin{equation}
		\frac{1}{N}\sum_{j=1}^{N}\mathbb{E}_X[H_j(X)^2] \geq \frac{1}{4} \left (\mathbb{E}_j[b_j^2]+\sigma_X^2\frac{1}{N}\sum_{j=1}^{N}\sum_{i=1}^d\left(W_1^{ij}\right)^2  \right),
	\end{equation} which implies
	\begin{equation}\label{eq:main}
		\frac{1}{N}\sum_{j=1}^{N}4\mathbb{E}_X[H_j(X)^2] \geq \sigma_W^2 + d\sigma_X^2\mathbb{E}_{i,j}\left[\left(W_1^{ij}\right)^2\right].
	\end{equation}
	Now, originally, the unpruned weights in $W_1$ follow the given $\mathcal{N}(0,\sigma_W^2I)$ distribution. However, after $k$ cycles of pruning, the smallest $1-(1-p)^k$ proportion of weights in $W_1$ get removed. Thus, the distribution $P(W_1)$ changes such that, for some appropriate $\beta$ (which depends on $k$), all $P(-\beta \leq W_1\leq \beta)=0$ except for $P(W_1=0)$, which will follow $P(W_1=0)=1-(1-p)^k$. Let us denote this modified distribution of weights via $\mathcal{N}_{p,k}(0,\sigma_W^2)$. Let us denote $P(W)=\mathcal{N}_{p,k}(0,\sigma_W^2)$ Note that as the hidden layer has infinite nodes ($N=\infty$) we can write
	\begin{align}
		\mathbb{E}_{i,j}\left[\left(W_1^{ij}\right)^2\right]=&\mathbb{E}_{W\sim\mathcal{N}_{p,k}(0,\sigma_W^2)}\left[W^2\right]\\
		=& \int_{-\infty}^{\infty} W^2P(W)dW\\
		=&  \int_{-\infty}^{-\beta} W^2dW +\int_{-\beta}^{\beta} W^2P(W)dW + \int_{\beta}^{\infty} W^2P(W)dW \\
		=& \int_{-\infty}^{-\beta} W^2P(W)dW + \int_{\beta}^{\infty} W^2P(W)dW \\ 
		=&\int_{-\infty}^{\infty} W^2P'(W)dW, \label{eq:w_nor}
	\end{align}
	where $P'(W)=P(W)Sign(W)$, where $Sign(W)=0$ for $-\beta\leq W\leq\beta$ and $Sign(W)=1$ otherwise.  
	In order to find $\int_{-\infty}^{\infty} W^2P'(W)dW$, for $w\sim\mathcal{N}(0,\sigma_W^2)$, we denote $Q(W)=\mathcal{N}(0,\sigma_W^2)$, and we can write
	\begin{align}
		\mathbb{E}_{W\sim\mathcal{N}(0,\sigma_W^2)}\left[W^2\right]=&\sigma_W^2= \int_{-\infty}^{\infty} W^2Q(W)dW\\
		=&\int_{-\beta}^{\beta} W^2Q(W)dW + \int_{-\infty}^{\infty} W^2P'(W)dW\\
		=&\tau\int_{-\beta}^{\beta} W^2\frac{Q(W)}{\tau}dW + \int_{-\infty}^{\infty} W^2P'(W)dW. \\ 
		=&\tau\int_{-\infty}^{\infty} W^2{Q'(W)}dW + \int_{-\infty}^{\infty} W^2P'(W)dW.
	\end{align}
	Here, $\tau$ is computed such that $\int_{-\beta}^{\beta}{\frac{Q(W)}{\tau}}=1$, and $Q'(W)=Q(W)(1-Sign(W))/\tau$. Note that here, $\tau=1-(1-p)^k$. Furthermore, we can see that $Q'(W)$ represents the truncated normal distribution, which has a variance of $\sigma_W^2\left(1-\frac{2\beta}{\sqrt{2\pi}\sigma_W\tau}e^{-\frac{\beta^2}{2\sigma_W^2}}\right)$. We also note that as $\tau = \int_{-\beta}^{\beta}Q(W)$, and as $Q(W)$ is Gaussian, we can write, $\tau = erf(\frac{\beta}{\sqrt{2}\sigma_W})$, which also implies $\frac{\beta}{\sqrt{2}\sigma_W}=\erfi(\tau)$. Thus, we have 
	\begin{align}
		\int_{-\infty}^{\infty} W^2P'(W)dW =&\sigma_W^2-\tau\int_{-\infty}^{\infty} W^2{Q'(W)}dW\\
		=&\sigma_W^2 - \sigma_W^2\left(\tau - \frac{2\beta}{\sqrt{2\pi}\sigma_W}e^{-\frac{\beta^2}{2\sigma_W^2}} \right)\\
		=& \sigma_W^2\left(1 - \tau +\frac{2\beta}{\sqrt{2\pi}\sigma_W}e^{-\frac{\beta^2}{2\sigma_W^2}} \right)\\
		=& \sigma_W^2\left((1-p)^k +\frac{2\erfi(1-(1-p)^k)}{\sqrt{\pi}}e^{-\left(\erfi(1-(1-p)^k)\right)^2} \right).
	\end{align}
	Combined with \eqref{eq:main} and \eqref{eq:w_nor}, we obtain, 
	\begin{align}
		\frac{1}{N}\sum_{j=1}^{N}&4\mathbb{E}_X[H_j(X)^2] \geq \sigma_W^2 + d\sigma_X^2\mathbb{E}_{i,j}\left[\left(W_1^{ij}\right)^2\right]\\
		=&\sigma_W^2 + d\sigma_X^2\sigma_W^2\left((1-p)^k +\sqrt{\frac{4}{\pi}}\erfi(1-(1-p)^k)e^{-\left(\erfi(1-(1-p)^k)\right)^2} \right),
	\end{align}
	\sloppy
	which yields the result, as $E_{AA}(H)=\mathbb{E}_X\left[\frac{1}{N}\sum_{j=1}^{N}H_j(X)^2\right]=\frac{1}{N}\sum_{j=1}^{N}\mathbb{E}_X[H_j(X)^2]$. 
\end{proof}

\textbf{Extension of Theorem 1 to the arbitrary depth case:}

\begin{corollary}
	We consider the same setting as in Theorem 1, except for the fact that we consider neural networks of arbitrary depth $D$. Furthermore, let each hidden layer contain $M$ neurons. We specify the distributions of the weights for each layer as follows: for the first layer we have $W_1\sim\mathcal{N}(0,\sigma_{\scaleto{W}{3pt}}^2 I)$, and for all subsequent layers ($l>1$), we have  $W_l\sim\mathcal{N}(0,\frac{1}{\sqrt{M}}\sigma_{\scaleto{W}{3pt}}^2 I)$. Furthermore, we consider the case where all the first $D-1$ layers do not have biases associated with the weights, but only the $D^{th}$ layer has biases, which has the same distribution as weights, like before. Now, same as in Theorem 1, we consider an iterative pruning method, where in each iteration a fraction $0\leq p\leq1$  of the smallest magnitude weights are pruned (layer-wise pruning). Let $H_D$ be the hidden layer output at a depth of $D$. Then, after $k$ iterations of pruning, in the limiting case of $M\rightarrow \infty$ it holds that
	\begin{equation} \label{eq:arb}
		\resizebox{0.9\hsize}{!}{
			$4E_{AA}(H_D)\geq\sigma_W^2 + \frac{1}{2^{D-1}}d\sigma_X^2\sigma_W^{2D}\left((1-p)^k +\sqrt{\frac{4}{\pi}}\erfi\left(1-(1-p)^k\right)e^{-\left(\erfi\left(1-(1-p)^k\right)\right)^2} \right)$}
	\end{equation}
	where $\erfi(.)$ is the inverse error function \cite{andrews1998special}. 
\end{corollary}
\begin{proof}
    To prove this result, we first note that, as we are considering the limiting case of $M\rightarrow \infty$, the function at any hidden neuron $H_i$ at depth $D$ can always be represented as follows 
    \begin{equation}
        H_j(X) = max\left(0,b_j+\sum_{i=1}^dW_{eff}^{ij}X_i\right),
    \end{equation}
    where $W_{eff}^{ij}$ represents the \textit{effective} weight random variables from the first $D-1$ layers of the network which is associated with $X_i$. Note that the non-linearties associated with ReLU activations in the network will be subsumed inside of $W_{eff}^{ij}$ using other random variables which are probabilistically $0$ or $1$. We elaborate on this later. 
    
    We note that due to $M\rightarrow \infty$, we can assume $W_{eff}$ to be normally distributed. Furthermore, due to symmetry of distribution and computation, it is clear that for all $i$ and $j$, $W_{eff}^{ij}$ will have the same distribution parameters, which we denote by $\mathcal{N}\left(\mu_{eff},\sigma^2_{eff}\right)$. Furthermore, as there are no biases associated with the first $D-1$ layers of computation, we have that $\mu_{eff}=0$.
    
    To estimate $\sigma^2_{eff}$, we first show the estimation of $\sigma^2_{eff}$ at depth $D=2$. For $D=2$, note that the effective weight that is tied to the input $X_1$ ($i=1$) and the output hidden node $j=1$, can be written as follows:
    \begin{equation}
        W_{eff}^{11} = \sum_{k=1}^{M}W_1^{1k}W_2^{k1}\delta_k,
    \end{equation}
    
    where $\delta_1,\delta_2,...\delta_M$ are random variables which are associated with the ReLU non-linearity at the output of the first hidden layer. Although $\delta_1,\delta_2,...\delta_M$ are dependent on the output at the hidden nodes of the first hidden layer, we note that w.r.t $X_1$ and the weights $W_1^{1k}$ and $W_2{k1}$ themselves, they are independent. Furthermore, as $M\rightarrow \infty$, we can therefore consider $\delta_1,\delta_2,...\delta_M$ to be independent random variables distributed as $P(\delta_i=0)=P(\delta_i=1)=0.5$. This follows from the fact that $X_i$ and $W_1$ both are normally distributed with zero-mean. With this, we can estimate $\sigma^2_{eff}$ for $D=2$ as follows: 
    
    \begin{align}
        \sigma^2_{eff} &= \mathbb{E}\left[(\sum_{k=1}^{M}W_1^{1k}W_2^{k1}\delta_k)^2\right]= \mathbb{E}\left[\sum_{k=1}^{M}(W_1^{1k}W_2^{k1}\delta_k)^2\right]\\
        &=\sum_{k=1}^{M}\mathbb{E}\left[(W_1^{1k}W_2^{k1}\delta_k)^2\right]
        =\sum_{k=1}^{M}\frac{1}{2} \sigma_W^2 \frac{\sigma_W^2}{M}=\frac{1}{2}\sigma^4_W
    \end{align}
    
    Similarly, one can easily generalize the above result to the general case of depth $D$ as shown below. For the general case of depth $D$, we will have that 
    \begin{equation}
            \sigma^2_{eff} = \left(\frac{1}{2}\right)^{D-1}\sigma_W^{2D}
    \end{equation}
    Note that with this, we can indeed directly apply the result in Theorem 1, considering the new values of the energy of the effective weights, to obtain: 
    
    \begin{equation} 
		\resizebox{0.9\hsize}{!}{
			$4E_{AA}(H_D)\geq\sigma_W^2 + \frac{1}{2^{D-1}}d\sigma_X^2\sigma_W^{2D}\left((1-p)^k +\sqrt{\frac{4}{\pi}}\erfi\left(1-(1-p)^k\right)e^{-\left(\erfi\left(1-(1-p)^k\right)\right)^2} \right)$}
	\end{equation}
    
    This proves our intended result. 
\end{proof}

\begin{theorem} 
	In Theorem \ref{thm:1}'s setting, we consider a single epoch of weight update for the network across a training dataset $S=\{(X_1,Y_1),..,(X_n,Y_n)\}$ using the cross-entropy loss, where $Y_i\in\{0,1\}$. Let $\alpha$ denote the learning rate. Let us denote the R.H.S of \eqref{eq:thm1} by $C(\sigma_X,\sigma_W,p,k)$. Let the final layer weights before and after one training epoch be $W_2$ and $W_2'$ respectively. We have, 
	\begin{equation} \label{eq:wg_final}
		\mathop{{}\mathbb{E}}_{W_2\sim\mathcal{N}_k(0,\sigma_{\scaleto{W}{3pt}}^2 I)}\left[E_{WG}\left(W_2,W_2'\right)\right] \geq \alpha^2\gamma C(\sigma_X,\sigma_W,p,k),
	\end{equation}
	where $\mathcal{N}_k(0,\sigma_{\scaleto{W}{3pt}}^2 I)$ represents a normal distribution $\mathcal{N}(0,\sigma_{\scaleto{W}{3pt}}^2 I)$ initialization of $W_2$, followed by $k$ iterations of pruning, and $\gamma$ is a constant.
\end{theorem}

{\bf Additional Motivations for the S-shape. }Theorem \ref{thm:2} implies that 
when $C(\sigma_X,\sigma_W,p,k)$ decreases during iterative pruning, the learning rate $\alpha$ has to increase proportionally so that the lower bound for the {\em gradient energy} ($E_{WG}$) remains constant for all pruning cycles. We plot the adapted learning rate $\alpha$ (i.e., $\alpha \propto \sqrt{K/\gamma C(\sigma_X,\sigma_W,p,k)}$) with pruning rate $p$ = 0.2 and pruning cycles $k$ = 25, where $K = \mathop{{}\mathbb{E}}_{W_2\sim\mathcal{N}_k(0,\sigma_{\scaleto{W}{3pt}}^2 I)}[E_{WG}(W_2,W_2')]$. We find out that the adapted learning rate $\alpha$ resembles a S-shape trajectory, motivating the proposed SILO.

\begin{proof}
	For the $k^{th}$ sample, let $(a_0^k,a_1^k)$ denote the network output probabilities for the two classes. Thus, we have $a_0^k+a_1^k=1$ for all $k$. Furthermore, for the $k^{th}$ sample, let $(y_0^k,y_1^k)$ represent the one-hot label output, which will depend on the true label $Y_k$. For the $k^{th}$ sample, let $(z_0^k,z_1^k)$ denote the network output logits, from which the output probabilities are computed using the softmax operator. Let $W_2^{ij}$ represent the weight that connects the $i^{th}$ hidden node $H_i$ to the $j^{th}$ output node. Let $L(X_k,Y_k) = -\sum_{l=1}^2 y^k_l \ln{a^k_l}$ be the cross-entropy loss for the $k^{th}$ sample. Lastly, let $h_i^k$ represent the output of the $i^{th}$ node in $H$, for the $k^{th}$ sample. Using backpropagation, the weight update for $W_2^{ij}$, from a single example $X_k$, can be written as, 
	\begin{align}
		W_2^{ij} =& W_2^{ij} -\alpha \frac{\partial L(X_k,Y_k)}{\partial W_2^{ij}}\\
		=& W_2^{ij} - \alpha \sum_{l=1}^{2} \frac{\partial L(X_k,Y_k)}{\partial a_l^k}\frac{\partial a_l^k}{\partial  z_j^k}\frac{\partial z_j^k}{\partial  W_2^{ij}} \\
	\end{align}
	It can be shown that $\frac{\partial L(X_k,Y_k)}{\partial a_l^k} = -\frac{y^k_l}{a_l^k}$, and $\frac{\partial a_l^k}{\partial z_l^k} = a_l^k(1-a_l^k)$. For $l \neq j$, we can show that $\frac{\partial a_l^k}{\partial z_j^k} = -a_l^ka_j^k$.  Furthermore, we also have that $\frac{\partial z_j^k}{\partial  W_2^{ij}} = h_i^k$. Combining these results eventually yields, 
	\begin{equation} \label{eq:w_upd}
		W_2^{ij} = W_2^{ij} -\alpha h_i^k\left(a_j^k-y_j^k \right).
	\end{equation}
	Note that this is the update for a single example. Iterating through the entire dataset, the final value of the updated $W_2^{ij}$ can be written as,
	\begin{equation}
		W_2^{ij} = W_2^{ij} -\alpha \sum_{k=1}^n h_i^k\left(a_j^k-y_j^k \right).
	\end{equation}
	Using the above we can write the total difference of squares between all weights, $\mathbb{E}[\lVert W_2-W_2'  \rVert^2]$, as 
	\begin{equation}
		\mathbb{E}[\lVert W_2-W_2'  \rVert^2] = \alpha^2\frac{1}{N} \sum_{i=1}^{N} \sum_{l=1}^2 \left(\sum_{k=1}^{n} h_i^k\left(a_l^k-y_l^k \right)\right)^2 
	\end{equation}
	Note that here $N=\infty$, as mentioned in Theorem 1. 
	The above expression can be split into two terms as follows
	\begin{align}
		\mathbb{E}[\lVert W_2-W_2'  \rVert^2] =&\alpha^2 \frac{1}{N}\sum_{i=1}^{N} \sum_{l=1}^2 \left(\sum_{k=1}^{n} h_i^k\left(a_l^k-y_l^k \right)\right)^2 \\
		=&\alpha^2\sum_{i=1}^{N} \sum_{l=1}^2 \sum_{k=1}^{n} (h_i^k)^2\left(a_l^k-y_l^k \right)^2 +  \\
		& \alpha^2\frac{1}{N}\sum_{i=1}^{N} \sum_{l=1}^2 \sum_{k_1=1}^{n}\sum_{k_2=1}^{n} h_i^{k_1} h_i^{k_2}\left(a_l^{k_1}-y_l^{k_1} \right)\left(a_l^{k_2}-y_l^{k_2} \right) \label{eq:twos}
	\end{align}
	
	We analyze each term separately as follows, starting with the first expression. In what follows, we incorporate the expectation over $\mathop{{}\mathbb{E}}_{W_2\sim\mathcal{N}_k(0,\sigma_{\scaleto{W}{3pt}}^2 I)}$ into the two expressions. For simplicity of notation, the expectation $\mathop{{}\mathbb{E}}_{W_2\sim\mathcal{N}_k(0,\sigma_{\scaleto{W}{3pt}}^2 I)}$ will be written simply as $\mathop{{}\mathbb{E}}_{W_2}$.
	\begin{align}
		\mathop{{}\mathbb{E}}_{W_2}\left [ \frac{1}{N}\sum_{i=1}^{N} \sum_{l=1}^2 \sum_{k=1}^{n} (h_i^k)^2\left(a_l^k-y_l^k \right)^2 \right] = 2n\times \mathop{\mathbb{E}}_{W_2,i,l,k}  \left [ (h_i^k)^2\left(a_l^k-y_l^k \right)^2 \right ], 
	\end{align}
	Now, let us define two random variables $R_1$ and $R_2$, such that $R_1 = (h_i^k)^2$ and $R_2 = \left(a_l^k-y_l^k \right)^2$  where $0\leq i \leq N$, $0\leq k \leq n$ and $1\leq l \leq 2$ are random variables all drawn uniformly within their corresponding range. We note that $\mathbb{E}_{W_2,i,l,k}  \left [ (h_i^k)^2\left(a_l^k-y_l^k \right)^2 \right ] = \mathbb{E}_{W_2,i,l,k}\left[R_1R_2 \right]$. Note that the random variable $h_i^k$ is independent of $a_l^k$, as $h_i^k$ does not yield any information about $W_2$, and as $a_l^k$ is a function of $H^k$ and $W_2$, this implies $P(a_l^k|h_i^k)=P(a_l^k)$. The independence of $h_i^k$ and $y_l^k$ follows in the same manner. Thus, as $\mathbb{E}[XY]=\mathbb{E}[X]\mathbb{E}[Y]$ for independent $X$ and $Y$, we have
	\begin{align}
		\mathop{{}\mathbb{E}}_{W_2}\left [\frac{1}{N}\sum_{i=1}^{N} \sum_{l=1}^2 \sum_{k=1}^{n} (h_i^k)^2\left(a_l^k-y_l^k \right)^2\right] =& 2n \times \mathop{\mathbb{E}}_{W_2,i,l,k} \left[ (h_i^k)^2\left(a_l^k-y_l^k \right)^2\right]\\
		=&2n\mathop{\mathbb{E}}_{W_2,i,k}  \left[(h_i^k)^2\right] \mathop{\mathbb{E}}_{W_2,l,k} \left[\left(a_l^k-y_l^k \right)^2 \right]\\
		\geq& n C(\sigma_X,\sigma_W,p,k) \gamma',
	\end{align}
	
	where the last step follows from Theorem \ref{thm:1}, and  $\gamma'=\mathbb{E}_{W_2,l,k} \left[\left(a_l^k-y_l^k \right)^2 \right]/2$ is a constant that only depends on the first layer weights $W_1$, as the expectation is over all $W_2$. Similarly, we can expand the second term of \eqref{eq:twos} as follows. 
	\begin{align}
		\mathop{{}\mathbb{E}}_{W_2}&\left [\frac{1}{N}\sum_{i=1}^{N} \sum_{l=1}^2 \sum_{k_1=1}^{n} \sum_{k_2=1}^{n} h_i^{k_1} h_i^{k_2}\left(a_l^{k_1}-y_l^{k_1} \right)\left(a_l^{k_2}-y_l^{k_2} \right) \right]\\
		&\quad \quad \quad \quad= 2n^2 \times \mathop{\mathbb{E}}_{W_2,i,l,k_1,k_2} \left[h_i^{k_1} h_i^{k_2} \left(a_l^{k_1}-y_l^{k_1} \right)\left(a_l^{k_2}-y_l^{k_2} \right) \right].
	\end{align}
	\sloppy
	As before, we define two random variables in this context, $R_1=h_i^{k_1} h_i^{k_2}$ and $R_2=\left(a_l^{k_1}-y_l^{k_1} \right)\left(a_l^{k_2}-y_l^{k_2} \right)$, where $0\leq i \leq N$, $0\leq k_1,k_2 \leq n$ and $1\leq l \leq 2$ are random variables all drawn uniformly within their corresponding range. We similarly note that $\mathbb{E}_{W_2,i,l,k_1,k_2} \left[h_i^{k_1} h_i^{k_2} \left(a_l^{k_1}-y_l^{k_1} \right)\left(a_l^{k_2}-y_l^{k_2} \right) \right] = \mathbb{E}_{W_2,i,l,k_1,k_2}\left[R_1R_2\right]$. Like before, $h_i^{k_1}$ individually is independent of $\left(a_l^{k_1}-y_l^{k_1} \right)$ and $h_i^{k_2}$ is independent of $\left(a_l^{k_2}-y_l^{k_2} \right)$, implying that $h_i^{k_1} h_i^{k_2}$ is independent of $\left(a_l^{k_1}-y_l^{k_1} \right)\left(a_l^{k_2}-y_l^{k_2} \right)$. Thus, we can write $\mathbb{E}_{W_2,i,l,k_1,k_2}\left[R_1R_2\right] = \mathbb{E}_{W_2,i,k_1,k_2}\left[R_1\right]\mathbb{E}_{W_2,l,k_1,k_2}\left[R_2\right]$. 
	
	Furthermore, we have that $\left(a_l^{k_1}-y_l^{k_1} \right)$ is itself independent of $\left(a_l^{k_2}-y_l^{k_2} \right)$, as $P(k_2|k_1)=P(k_2)$ as they are independently chosen. Thus, it similarly follows that $\mathbb{E}_{W_2,l,k_1,k_2}\left[R_2\right] = \mathbb{E}_{W_2,l,k_1}\left[\left(a_l^{k_1}-y_l^{k_1} \right)\right]\mathbb{E}_{W_2l,k_2}\left[\left(a_l^{k_2}-y_l^{k_2} \right)\right]$. Lastly note $\mathbb{E}_{W_2,l,k_1}\left[\left(a_l^{k_1}-y_l^{k_1} \right)\right]=\mathbb{E}_{W_2,k_1}\left[\left(a_0^{k_1}+a_1^{k_1}-y_0^{k_1}-y_1^{k_1} \right)\right]=\mathbb{E}_{W_2,k_1}\left[0\right]=0$. This results in, 
	\begin{align}
		\mathop{{}\mathbb{E}}_{W_2}&\left[\frac{1}{N}\sum_{i=1}^{N} \sum_{l=1}^2 \sum_{k_1=1}^{n} \sum_{k_1=1}^{n} h_i^{k_1} h_i^{k_2}\left(a_l^{k_1}-y_l^{k_1} \right)\left(a_l^{k_2}-y_l^{k_2} \right)\right] \\
		&= 2n^2 \times \mathop{\mathbb{E}}_{W_2,i,l,k_1,k_2} \left[h_i^{k_1} h_i^{k_2} \left(a_l^{k_1}-y_l^{k_1} \right)\left(a_l^{k_2}-y_l^{k_2} \right) \right]\\
		&=2n^2 \times \mathop{\mathbb{E}}_{W_2,i,k_1,k_2} \left[h_i^{k_1} h_i^{k_2}\right] \mathop{\mathbb{E}}_{W_2,l,k_1,k_2}\left[ \left(a_l^{k_1}-y_l^{k_1} \right)\left(a_l^{k_2}-y_l^{k_2} \right) \right]\\
		&=2n^2 \times \mathop{\mathbb{E}}_{W_2,i,k_1,k_2} \left[h_i^{k_1} h_i^{k_2}\right] \mathop{\mathbb{E}}_{W_2,l,k_1}\left[ \left(a_l^{k_1}-y_l^{k_1} \right)\right] \mathop{\mathbb{E}}_{W_2,l,k_2}\left[\left(a_l^{k_2}-y_l^{k_2} \right) \right]\\
		&= 0,
	\end{align}
	
	where the last step follows as $\mathbb{E}_{W_2,l,k_1}\left[\left(a_l^{k_1}-y_l^{k_1} \right)\right]=0$. Thus, replacing the terms in \eqref{eq:twos}, we have
	\begin{equation}
		\mathop{{}\mathbb{E}}_{W_2}\left [\mathbb{E}[\lVert W_2-W_2'  \rVert^2]\right] \geq n\alpha^2 C(\sigma_X,\sigma_W,p,k)\gamma'
	\end{equation}
	However, note here that the above expression results from \textit{all} the weights in each iteration. In truth, as a $(1-p)^k$ proportion of the weights remain in $W_2$, only a $(1-p)^k$ fraction of the weights in $W_2$ will be updated. This indicates that the pruning corrected value of $\mathbb{E}[\lVert W_2-W_2'  \rVert^2]$, denoted by $\mathbb{E}_{corr}[\lVert W_2-W_2'  \rVert^2]$ must be, 
	\begin{equation} \label{eq:pru_cr}
		\mathbb{E}_{corr}[\lVert W_2-W_2'  \rVert^2] = \mathop{\mathbb{E}}_{i,j}[\lVert W_2^{ij}-(W_2^{ij})'  \rVert^2\delta_{ij}],
	\end{equation}
	where $\delta_{ij}$ is a variable which will be $0$ if $W_2^{ij}=0$ (pruned) or will be $1$, depending on the weight magnitude itself. However, note that as the magnitude of the weight does not affect the magnitude of change for $W_2^{ij}$ in \eqref{eq:w_upd}, $\delta_{ij}$ in \eqref{eq:pru_cr} is independent of $W_2^{ij}-(W_2^{ij})'$, and thus we can write $\mathbb{E}_{corr}[\lVert W_2-W_2'  \rVert^2] =  \mathop{\mathbb{E}}_{i,j}[\lVert W_2^{ij}-(W_2^{ij})'  \rVert^2]\mathbb{E}[\delta_{ij}]=(1-p)^k\mathbb{E}[\lVert W_2-W_2'  \rVert^2]$. This also indicates that the pruning corrected weight-gradient energy $\mathbb{E}[\lVert W_2-W_2'  \rVert^2]$, denoted by $E_{WG}\left(W_2,W_2'\right)$ will be
	
	\begin{align}
		\mathop{{}\mathbb{E}}_{W_2}\left[E_{WG}\left(W_2,W_2'\right)\right] =& \frac{\mathop{{}\mathbb{E}}_{W_2}\left[\mathbb{E}_{corr}[\lVert W_2-W_2'  \rVert^2]\right]}{(1-p)^k} = \mathop{{}\mathbb{E}}_{W_2}\left[\mathbb{E}[\lVert W_2-W_2'  \rVert^2]\right] \\
		&\geq n\alpha^2 C(\sigma_X,\sigma_W,p,k)\gamma'=\alpha^2 C(\sigma_X,\sigma_W,p,k)\gamma,
	\end{align}
	where we substitute $\gamma = n\gamma'$, which yields our result. 
\end{proof}

{\bf S-Shape learning rate for arbitrary depth networks. }
Even for the general case of arbitrary depth $D$, as shown in Corollary 1, we find that the resulting trajectory still resembles an S-shape, as the expression which controls the average activation energy has a similar form (see \eqref{eq:arb}). We also note that Theorem 2's result directly applies to the general case discussed in Corollary 1 as well. We present the extension of Theorem 2 to the more general case of arbitrary depth in the following Corollary. 

\begin{corollary}
     \label{thm:2}
	In Corollary 1's setting, we consider a single epoch of weight update for the network across a training dataset $S=\{(X_1,Y_1),..,(X_n,Y_n)\}$ using the cross-entropy loss, where $Y_i\in\{0,1\}$. Let $\alpha$ denote the learning rate. Let us denote the R.H.S of \eqref{eq:arb} by $C_D(\sigma_X,\sigma_W,p,k)$. Let the final layer weights before and after one training epoch be $W_D$ and $W_D'$ respectively. We have, 
	\begin{equation}\label{eq:lr_energy}
		\resizebox{0.55\hsize}{!}{
		$\mathop{{}\mathbb{E}}_{W_D\sim\mathcal{N}_k(0,\sigma_{\scaleto{W}{3pt}}^2 I)}\left[E_{WG}\left(W_D,W_D'\right)\right] \geq \alpha^2\gamma C_D(\sigma_X,\sigma_W,p,k)$,}
	\end{equation}
	where $\mathcal{N}_k(0,\sigma_{\scaleto{W}{3pt}}^2 I)$ represents a normal distribution $\mathcal{N}(0,\sigma_{\scaleto{W}{3pt}}^2 I)$ initialization of $W_D$, followed by $k$ iterations of pruning, and $\gamma$ is a constant.
\end{corollary}
\begin{proof}
The proof directly follows from the fact that in the proof of Theorem 2, we only make use of the backpropagatory signals, and therefore the result holds independent of the number of layers before the final layer. Furthermore, as the result in Theorem 2 only depends on the average activation energy of the final hidden layer, we thus have $C_D(\sigma_X,\sigma_W,p,k)$ (from \eqref{eq:arb}) in the R.H.S instead of $C(\sigma_X,\sigma_W,p,k)$ which applies only to the $D=1$ case (single hidden layer). This completes the proof.  
\end{proof}
\section{Supplementary Experimental Results}

In the Appendix, we show some additional experimental results. Specifically,

\begin{enumerate}
	
	\item In Section \ref{A1}, we show the experimental results on the distribution of weight gradients and hidden representations using AlexNet, ResNet-20 \& VGG-19 via both structured and unstructured pruning methods.
	
	\item In Section \ref{A2}, we present the performance comparison between SILO and selected LR schedule benchmarks using Adam \cite{kingma2014adam} and RMSProp \cite{tieleman2012lecture}. 
	
	\item In Section \ref{A3}, we show the performance comparison between SILO's $\texttt{max\_lr}$ to that of an Oracle using ResNet-20 with global gradient on CIFAR-10. 
	
	\item In Section \ref{A4}, we show the experimental results for more values of $\lambda$ for experiments in Tables \ref{per1} - \ref{per5}.
	
\end{enumerate}

\subsection{More Experimental Results on the Distribution of Weight Gradients and Hidden Representations}
\label{A1}

In this subsection, we present more experimental results on the distribution of weight gradients and hidden representations using popular networks and pruning methods in Figs. \ref{mlp_distribution_norm} - \ref{vgg}. The configuration for each network is given in Table \ref{arch_sec2}. 

We observe that the experimental results in Figs. \ref{mlp_distribution_norm} - \ref{vgg} largely mirror those in Fig. \ref{mlp_no_norm}(a). Specifically, in Fig. \ref{mlp_distribution_norm}, we show the distribution of weight gradients and hidden representations when iterative pruning a fully connected ReLU-based network using global magnitude with batch normalization applied to each hidden layer. We observe that the weight gradients of the unpruned network have a standard deviation of 1.2\texttt{e}-2 while the weight gradients of the pruned network ($\lambda$ = 13.4) have a much smaller standard deviation of 8\texttt{e}-3. Similarly, in Fig. \ref{alex}(a), where we show the distribution of weight gradients when iteratively pruning AlexNet, the standard deviation of the unpruned network ($\lambda$ = 100) is reduced from 1.2\texttt{e}-2 to 9\texttt{e}-3 when the network is iteratively pruned to $\lambda$ = 13.4. In Fig. \ref{resnet}(a), the weight gradients of the unpruned ResNet-20 ($\lambda$ = 100) have a standard deviation of 1.3\texttt{e}-2 while the weight gradients of the pruned ResNet-20 ($\lambda$ = 13.4) have a standard deviation of 8\texttt{e}-3. Moreover, in Fig. \ref{vgg}(a), the standard deviation of the unpruned VGG-19's weight gradients also reduces from 1.2\texttt{e}-2 to 9\texttt{e}-3 when the VGG-19 is iteratively pruned to $\lambda$ = 13.4.  Lastly, we note that the corresponding distributions of hidden representations are also shown in Fig. \ref{mlp_distribution_norm} (b) - \ref{vgg} (b), which largely mirror those in Fig. \ref{mlp_no_norm}(b).

{\renewcommand{\arraystretch}{1.2}
	\begin{table}[!ht]
		\small
		\setlength\tabcolsep{10.8pt} 
		\centering
		\begin{tabular}{l|ccccccccc} \hline \hline
			{Network}     & {Train Steps} & {Batch} & {Learning Rate Schedule} & {BatchNorm} \\ \hline
			\multirow{3}{*}{AlexNet}       &  781K Iters         &  64      &  warmup to 1\texttt{e}-2 over 150K, 10x drop at 300K, 400K              &  No     \\ \cmidrule{2-5}
			
			& \multicolumn{4}{c}{Pruning Metric: Unstructured Pruning - Layer Weight}\\ \hline
			\multirow{3}{*}{ResNet-20}  &  63K Iters            &128       &  warmup to 3\texttt{e}-2 over 20K, 10x drop at 20K, 25K   & Yes    \\ \cmidrule{2-5}
			& \multicolumn{4}{c}{Pruning Metric: Unstructured Pruning - Global Gradient}\\ \hline
			\multirow{3}{*}{VGG-19}      &   63K Iters           & 128      &  warmup to 1\texttt{e}-1 over 10K, 10x drop at 32K, 48K & Yes     \\ \cmidrule{2-5}
			& \multicolumn{4}{c}{Pruning Metric: Structured Pruning - L1 Norm}\\ \hline   \hline                         
		\end{tabular}
		\vspace{-3mm}
		\caption{Architectures and training details used in the appendix.}
		\vspace{-0mm}
		\label{arch_sec2}
\end{table}}

\begin{figure}[!h]
	\begin{minipage}{0.5\textwidth}
		\includegraphics[width=0.95\linewidth]{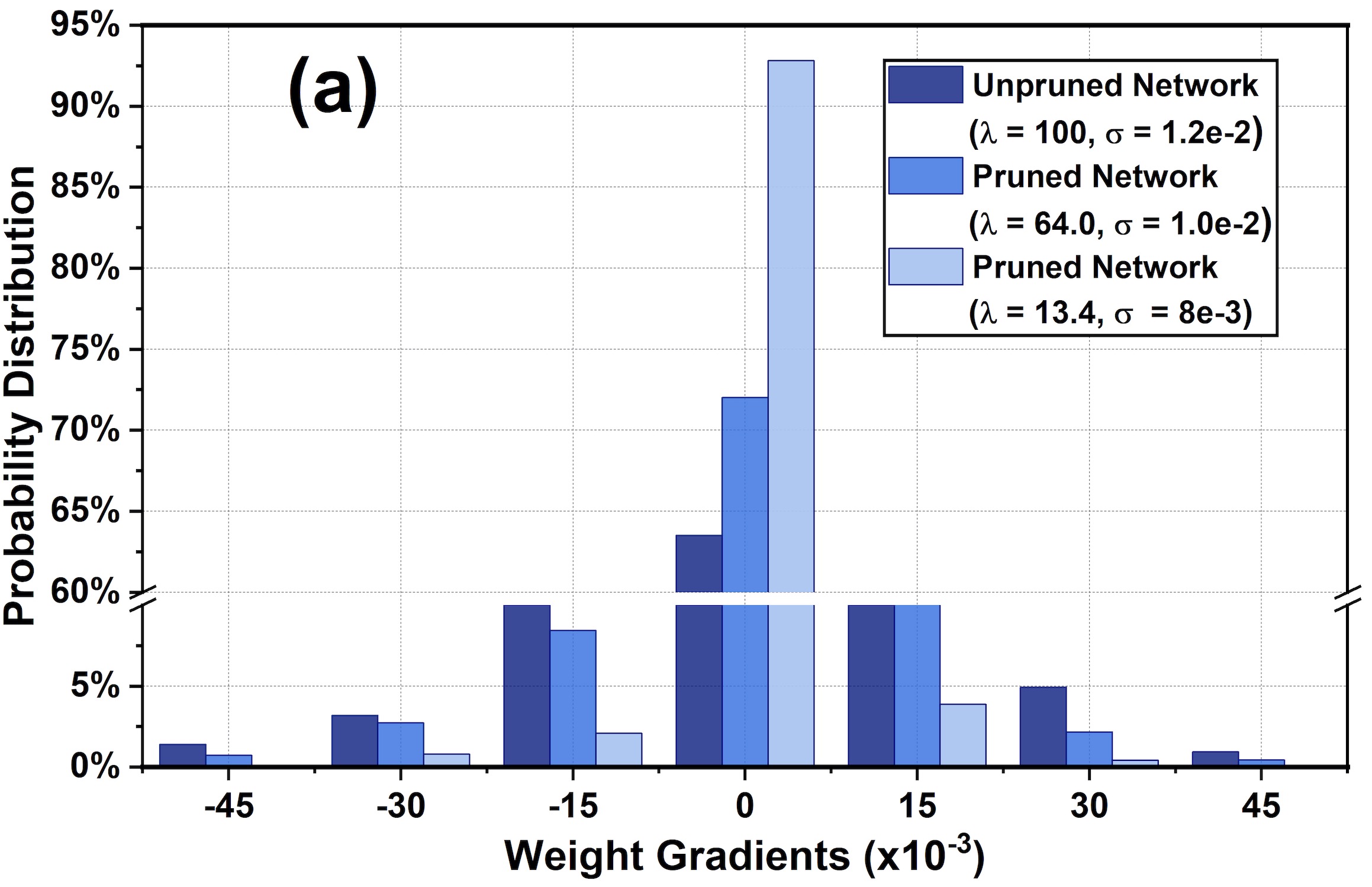}
	\end{minipage}%
	\begin{minipage}{0.5\textwidth}
		\includegraphics[width=0.95\linewidth]{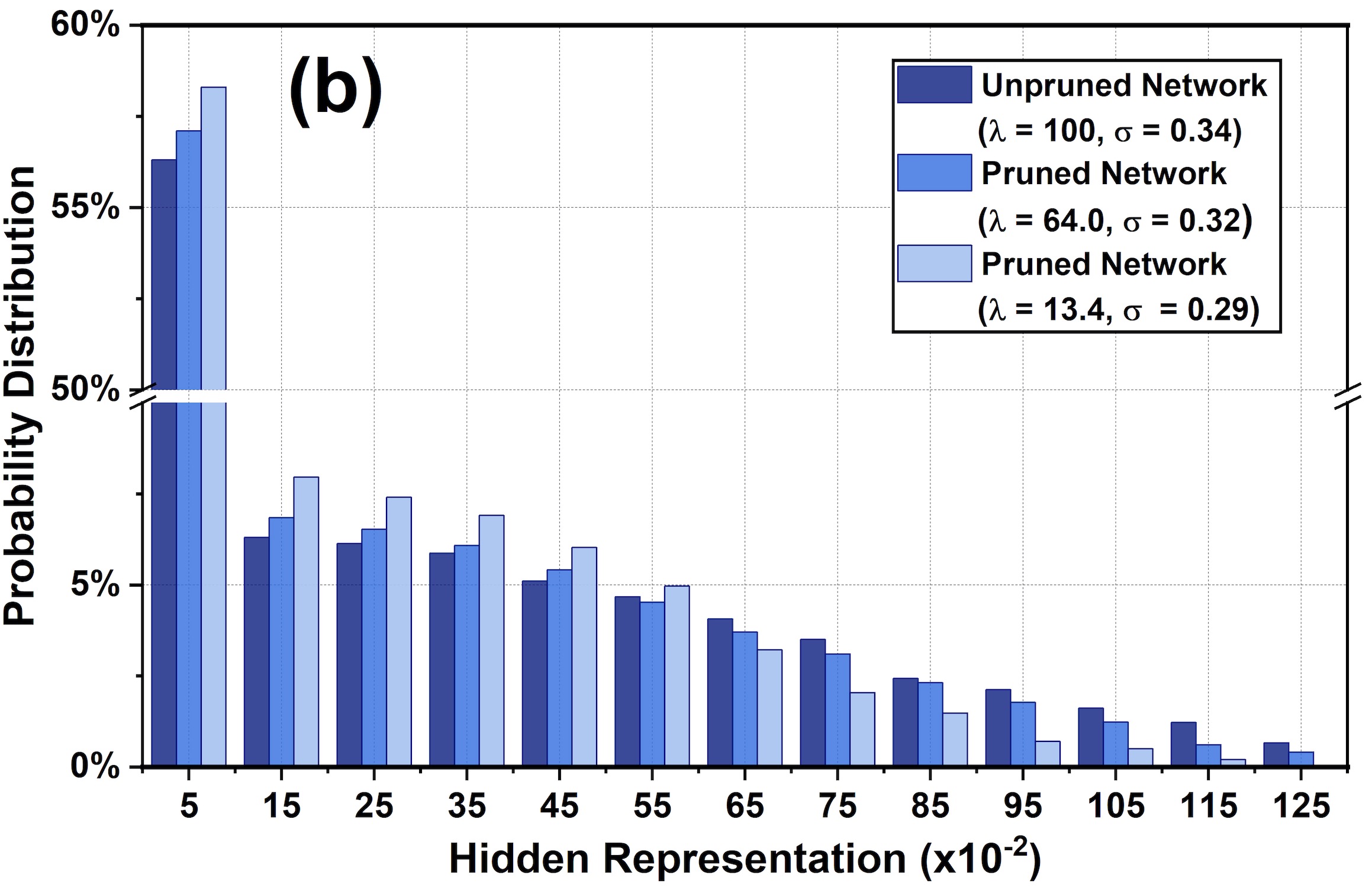}
	\end{minipage}
	\vspace{-3mm}
	\caption{(a) The distribution of weight gradients when iteratively pruning a fully connected ReLU-based network using global magnitude \cite{han2015learning}, where $\lambda$ is the percent of weights remaining and $\sigma$ is the standard deviation of the distribution. (b) The corresponding distribution of hidden representations (i.e., post-activation outputs of all hidden layers). The main difference with Fig. \ref{mlp_no_norm} is that the batch normalization is used here.} 
	\label{mlp_distribution_norm}
\end{figure}

\begin{figure}[!ht]
	\begin{minipage}{0.5\textwidth}
		\includegraphics[width=0.9\linewidth]{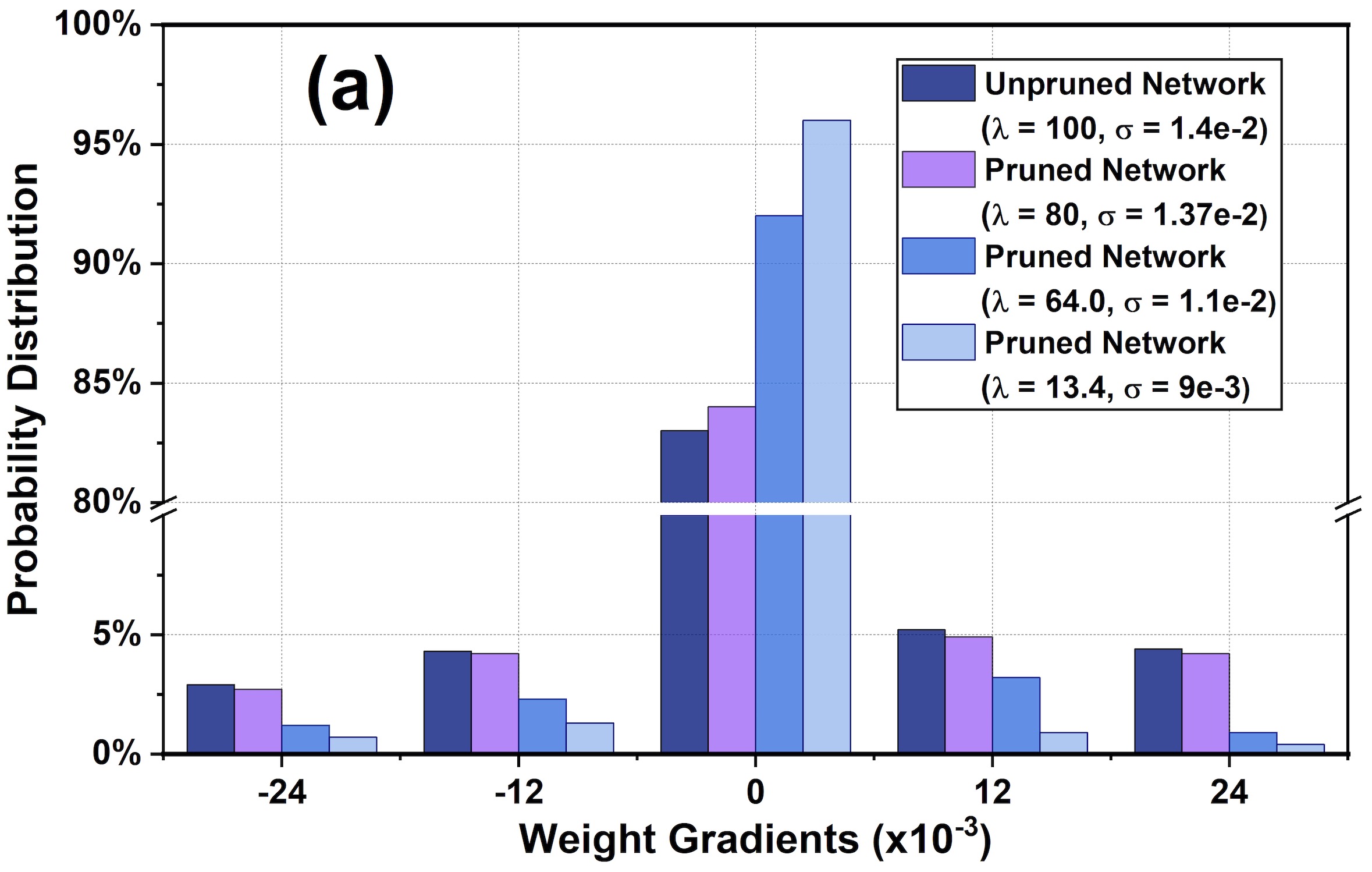}
	\end{minipage}%
	\begin{minipage}{0.5\textwidth}
		\includegraphics[width=0.9\linewidth]{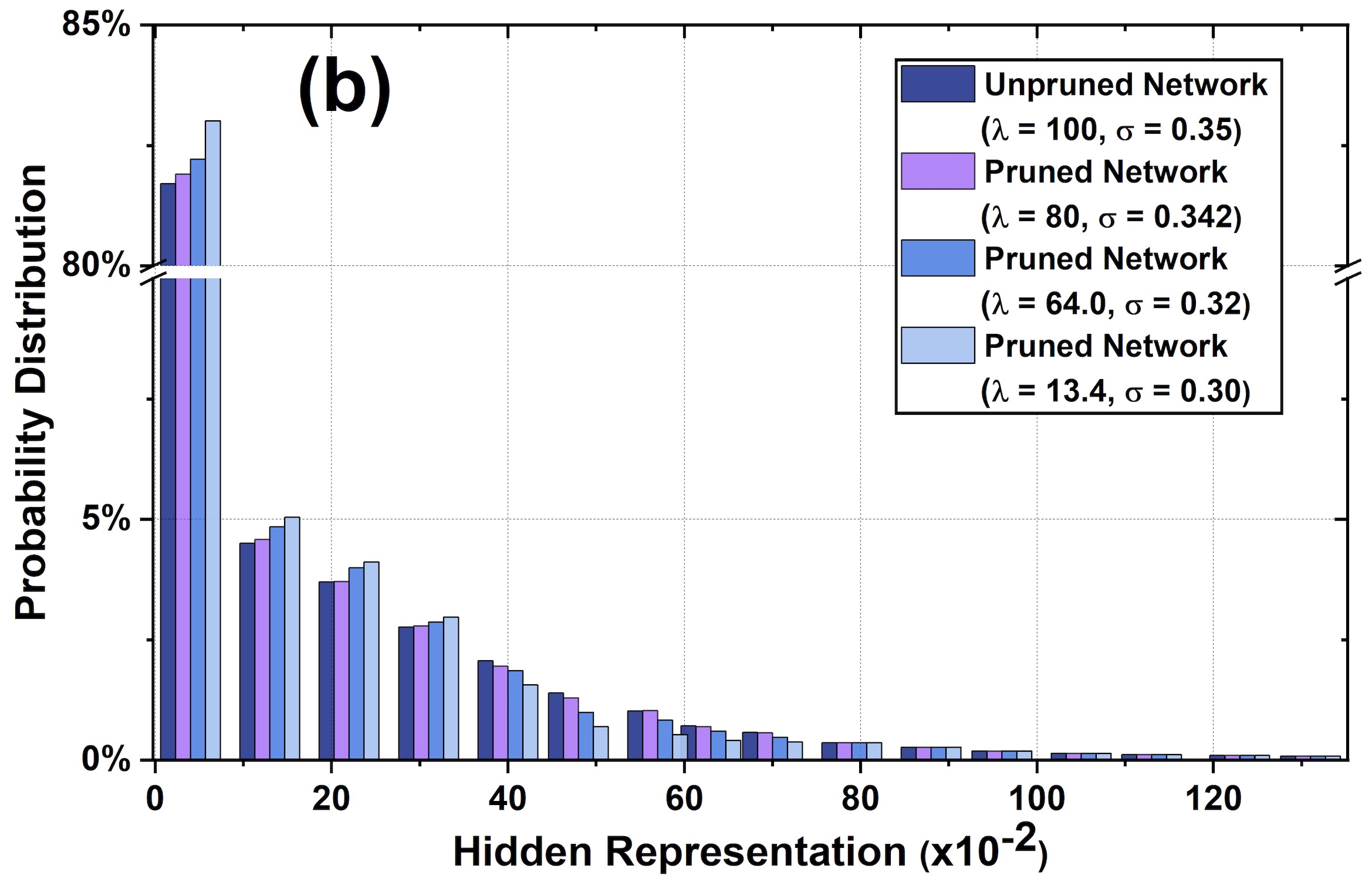}
	\end{minipage}
	\caption{(a) The distribution of weight gradients when iteratively pruning the AlexNet network using the layer magnitude pruning method, where $\lambda$ is the percent of weights remaining and $\sigma$ is the standard deviation of the distribution. (b) The corresponding distribution of hidden representations (i.e., post-activation outputs of all hidden layers). Please note that there is a line breaker in the vertical axis.} 
	\label{alex}
\end{figure}

\begin{figure}[!ht]
	\begin{minipage}{0.5\textwidth}
		\includegraphics[width=0.9\linewidth]{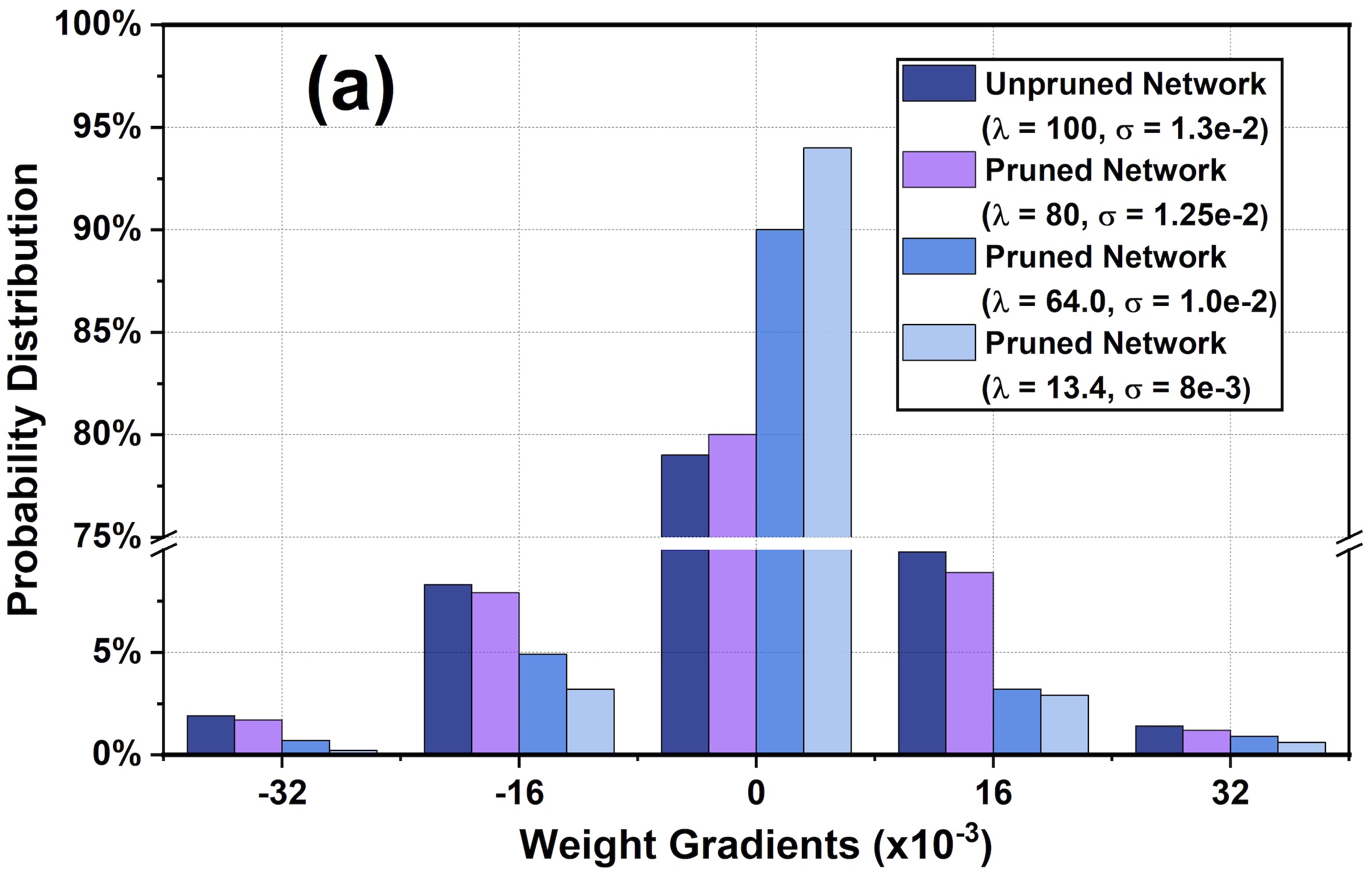}
	\end{minipage}%
	\begin{minipage}{0.5\textwidth}
		\includegraphics[width=0.9\linewidth]{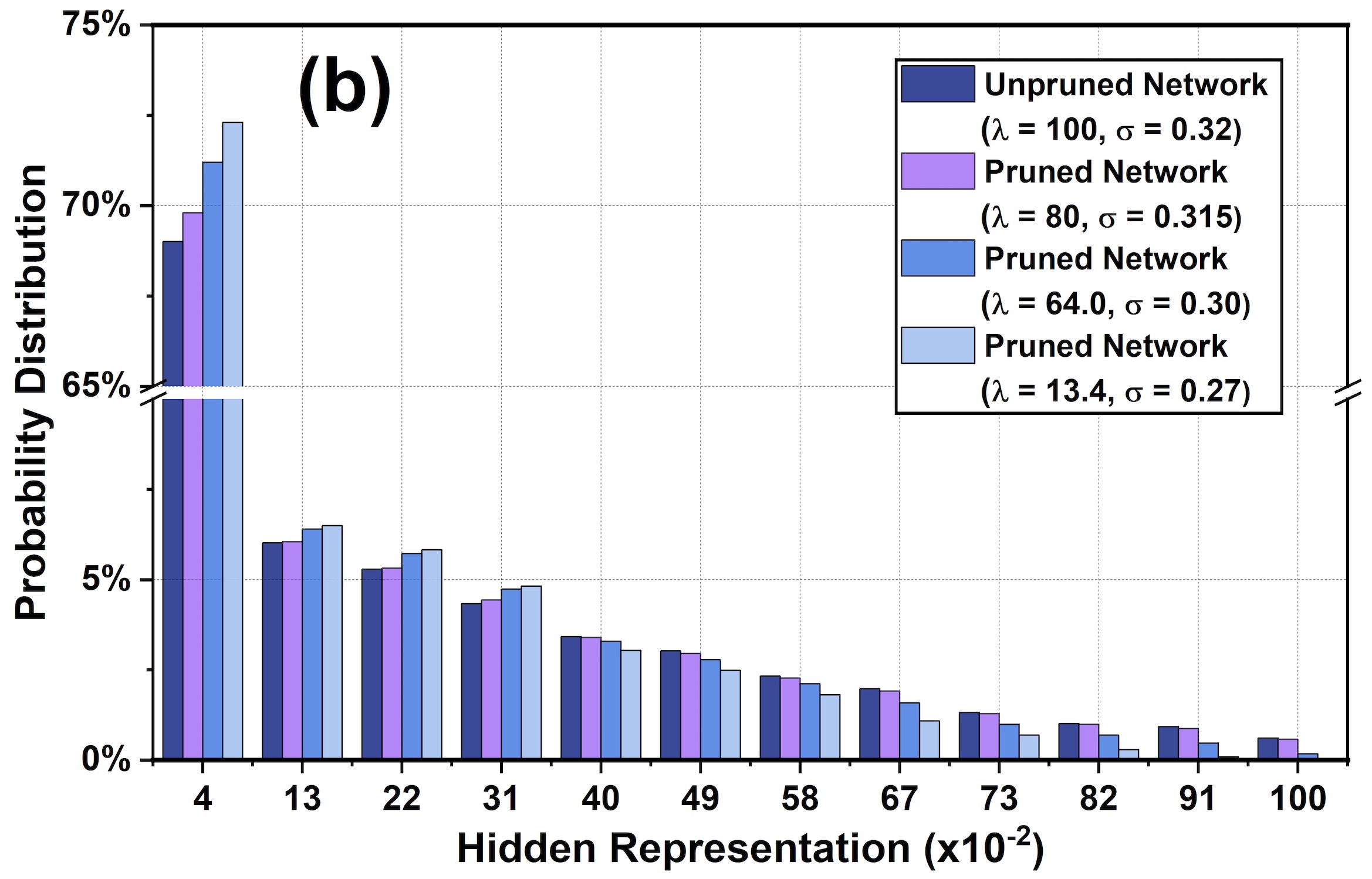}
	\end{minipage}
	\caption{(a) The distribution of weight gradients when iteratively pruning the ResNet-20 network using the global gradient pruning method, where $\lambda$ is the percent of weights remaining and $\sigma$ is the standard deviation of the distribution. (b) The corresponding distribution of hidden representations (i.e., post-activation outputs of all hidden layers). Please note that there is a line breaker in the vertical axis.} 
	\label{resnet}
\end{figure}

\begin{figure}[!ht]
	\begin{minipage}{0.5\textwidth}
		\includegraphics[width=0.9\linewidth]{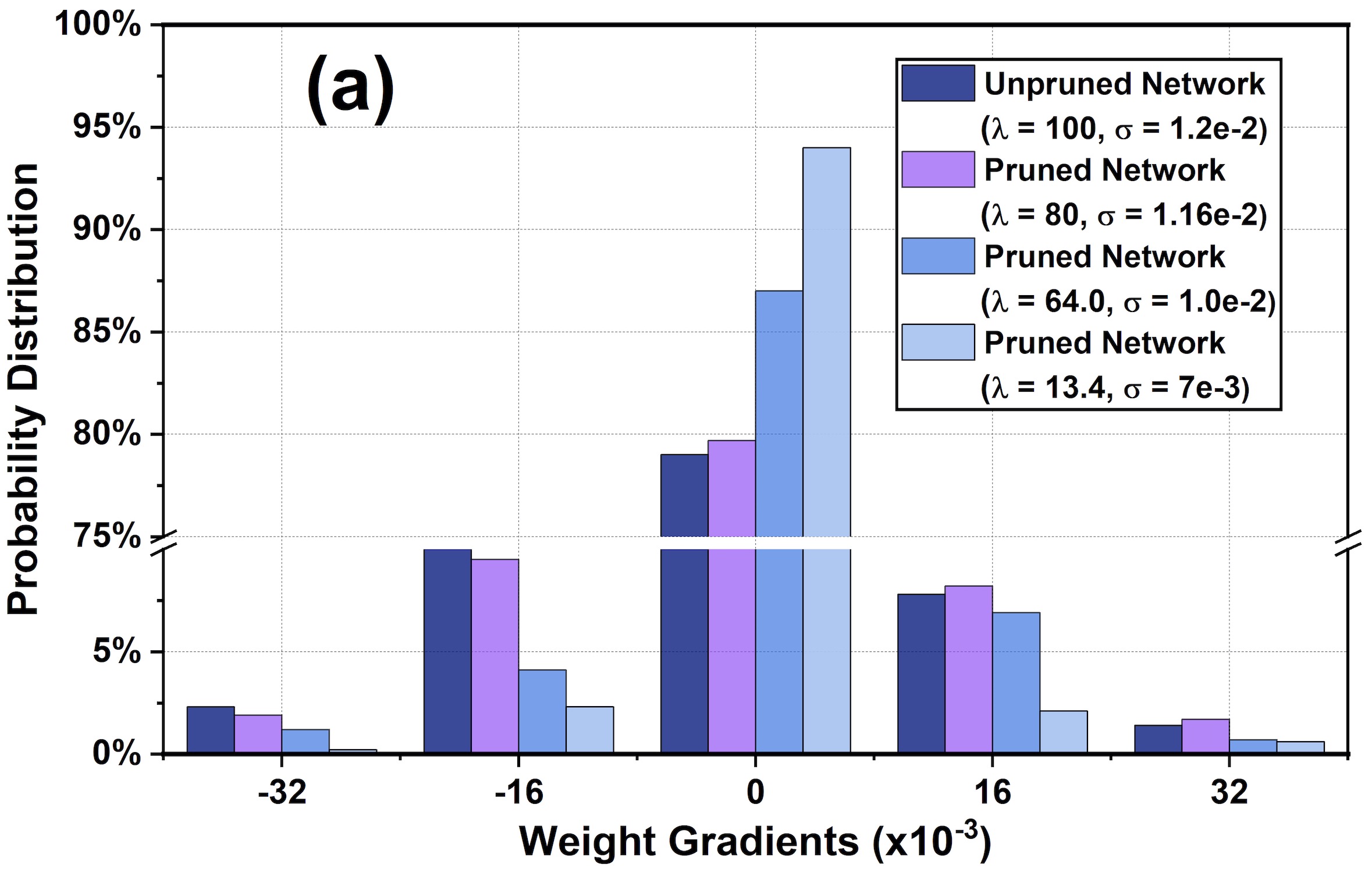}
	\end{minipage}%
	\begin{minipage}{0.5\textwidth}
		\includegraphics[width=0.9\linewidth]{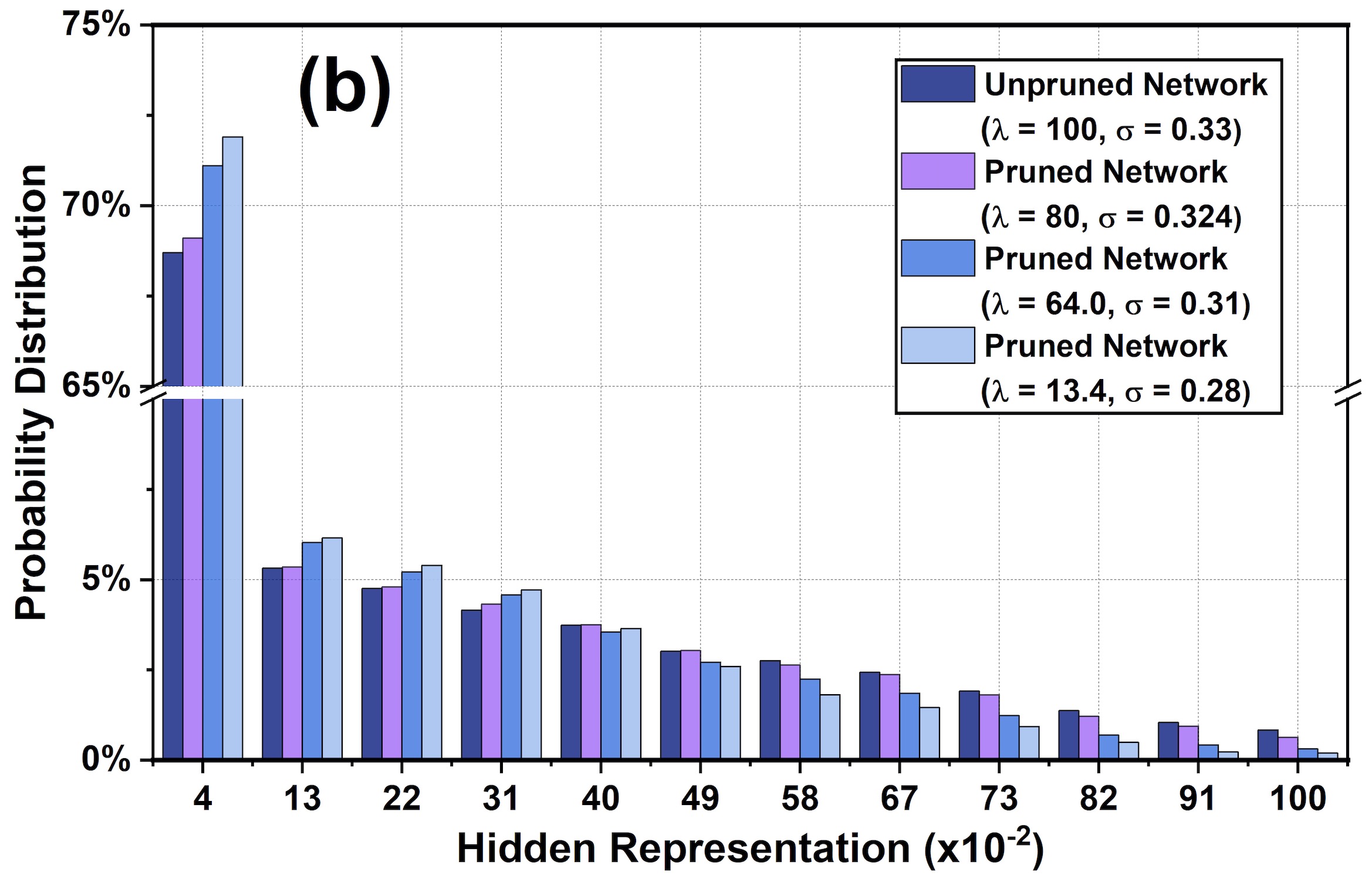}
	\end{minipage}
	\caption{(a) The distribution of weight gradients when iteratively pruning the VGG-19 network using the structured filter pruning method \cite{li2016pruning}, where $\lambda$ is the percent of weights remaining and $\sigma$ is the standard deviation of the distribution. (b) The corresponding distribution of hidden representations (i.e., post-activation outputs of all hidden layers). Please note that there is a line breaker in the vertical axis.} 
	\label{vgg}
\end{figure}

\clearpage
\newpage
\subsection{Performance Comparison using Adam and RMSProp}
\label{A2}

In this subsection, we show the performance comparison between the proposed SILO and selected LR benchmarks using Adam \cite{kingma2014adam} and RMSProp \cite{tieleman2012lecture} optimizers. The experimental results summarized in Tables \ref{per_adam} -  \ref{per_rms} largely mirror those in Table \ref{per3}. Specifically, the proposed SILO outperforms the best performing benchmark by a range of 0.8\% -2.7\% in pruned networks.

\begin{table}[!ht]
	\centering
	\setlength\tabcolsep{41.5pt}
	\begin{tabular}{l rrrrrr}
		\toprule
		\toprule
		\multicolumn{7}{c}{{\small {\bf Params}: 227K; {\small {\bf Train Steps}: 63K Iters; } {\bf Batch}: 128;  {\small {\bf Pruning Rate}: 0.2}}}\\ \toprule
		\multicolumn{7}{l}{{\small (1) constant LR} ({\footnotesize 8\texttt{e}-4}) {\small (2) LR decay} ({\footnotesize 3\texttt{e}-3, 63K})} \\
		\multicolumn{7}{l}{{\small (3) cyclical LR} ({\footnotesize 0, 3\texttt{e}-2, 8K})  (4) {\small LR-warmup} ({\footnotesize 3\texttt{e}-3, 20K, 20K, 25K, Nil})}\\
		\multicolumn{7}{l}{{\small {\small (5) SILO} ({\footnotesize 3\texttt{e}-3,4\texttt{e}-3,20K,20K,25K,Nil})}} \\
	\end{tabular}
	
	\setlength\tabcolsep{10.3pt}
	\begin{tabular}{lccccccccc}
		\toprule
		\multicolumn{1}{c}{$\lambda$} & 100 & 32.9 & 21.1 & 5.72 & 2.03 \\ \toprule
		{\small (1) constant LR} & 88.4$\pm${\scriptsize 0.4} & 84.8$\pm${\scriptsize 0.6} & 83.5$\pm${\scriptsize 0.6} & 75.5$\pm${\scriptsize 1.2} & 67.1$\pm${\scriptsize 1.7} \\
		{\small (2) LR decay} & 88.6$\pm${\scriptsize 0.3} & 87.1$\pm${\scriptsize 0.7} & 83.7$\pm${\scriptsize 0.9} & 76.1$\pm${\scriptsize 0.8} & 66.0$\pm${\scriptsize 1.3} \\
		{\small (3) cyclical LR}  & 88.9$\pm${\scriptsize 0.3} & 86.9$\pm${\scriptsize 0.5} & 84.1$\pm${\scriptsize 0.3} & 77.0$\pm${\scriptsize 0.9} & 64.4$\pm${\scriptsize 1.1} \\
		{\small (4) LR-warmup} & 89.1$\pm${\scriptsize 0.3} & 87.2$\pm${\scriptsize 0.4} & 84.5$\pm${\scriptsize 0.6} & 75.2$\pm${\scriptsize 1.1} & 65.1$\pm${\scriptsize 1.9} \\
		{\small (5) SILO} &  {\bf 89.2$\pm${\scriptsize 0.2}} & {\bf 87.9$\pm${\scriptsize 0.3}} & {\bf 86.3$\pm${\scriptsize 0.5}} & {\bf 79.5$\pm${\scriptsize 1.7}} & {\bf 71.7$\pm${\scriptsize 2.3}} \\
		\bottomrule
		\bottomrule
	\end{tabular}
	\vspace{-3mm}
	\caption{Performance comparison (averaged top-1 test accuracy $\pm$ std over 5 runs) of iteratively pruning ResNet-20 on CIFAR-10 dataset using the global magnitude pruning method \cite{han2015learning} and the Adam optimizer \cite{kingma2014adam}.}
	\label{per_adam}
\end{table}

\begin{table}[!ht]
	\centering
	\setlength\tabcolsep{40pt}
	\begin{tabular}{l rrrrrr}
		\toprule
		\toprule
		\multicolumn{7}{c}{{\small {\bf Params}: 227K; {\small {\bf Train Steps}: 63K Iters; } {\bf Batch}: 128;  {\small {\bf Pruning Rate}: 0.2}}}\\ \toprule
		\multicolumn{7}{l}{{\small (1) constant LR} ({\footnotesize 6\texttt{e}-4}) {\small (2) LR decay} ({\footnotesize 2\texttt{e}-3, 63K})} \\
		\multicolumn{7}{l}{{\small (3) cyclical LR} ({\footnotesize 0, 3\texttt{e}-2, 10K})  (4) {\small LR-warmup} ({\footnotesize 1\texttt{e}-3, 20K, 20K, 25K, Nil})}\\
		\multicolumn{7}{l}{{(5) {\small SILO} ({\footnotesize 1\texttt{e}-3,2\texttt{e}-3,20K,20K,25K,Nil})}} \\
	\end{tabular}
	
	\setlength\tabcolsep{10.3pt}
	\begin{tabular}{lccccccccc}
		\toprule
		\multicolumn{1}{c}{$\lambda$} & 100 & 32.9 & 21.1 & 5.72 & 2.03 \\ \toprule
		{\small (1) constant LR} & 87.9$\pm${\scriptsize 0.3} & 83.4$\pm${\scriptsize 0.4} & 81.5$\pm${\scriptsize 0.9} & 65.5$\pm${\scriptsize 1.9} & 55.1$\pm${\scriptsize 2.3} \\
		{\small (2) LR decay} & 88.4$\pm${\scriptsize 0.2} & 84.8$\pm${\scriptsize 0.6} & 77.8$\pm${\scriptsize 0.9} & 67.1$\pm${\scriptsize 1.4} & 58.3$\pm${\scriptsize 1.6} \\
		{\small (3) cyclical LR}  & 88.1$\pm${\scriptsize 0.3} & 84.7$\pm${\scriptsize 0.5} & 81.9$\pm${\scriptsize 0.7} & 67.5$\pm${\scriptsize 0.9} & 56.3$\pm${\scriptsize 1.7}   \\
		{\small (4) LR-warmup} & {\bf 88.9$\pm${\scriptsize 0.2}} & 85.1$\pm${\scriptsize 0.5} & 81.7$\pm${\scriptsize 0.4} & 67.3$\pm${\scriptsize 1.3} & 57.1$\pm${\scriptsize 1.4} \\
		{\small (5) SILO}   & 88.7$\pm${\scriptsize 0.3} & {\bf 86.1$\pm${\scriptsize 0.4}} & {\bf 83.1$\pm${\scriptsize 0.6}} & {\bf 72.5$\pm${\scriptsize 1.3}} & {\bf 63.5$\pm${\scriptsize 1.9}} \\
		\bottomrule
		\bottomrule
	\end{tabular}
	\vspace{-3mm}
	\caption{Performance comparison (averaged top-1 test accuracy $\pm$ std over 5 runs) of iteratively pruning ResNet-20 on CIFAR-10 dataset using the global magnitude pruning method \cite{han2015learning} and the RMSProp optimizer \cite{tieleman2012lecture}.}
	\label{per_rms}
\end{table}

\subsection{More Experimental Results on Comparing SILO to an Oracle}
\label{A3}

We show the performance between SILO's $\texttt{max\_lr}$ to that of an Oracle using ResNet-20 with global gradient on the CIFAR-10 dataset in Table \ref{lr_table_resnet_20}. It can be seen that the $\texttt{max\_lr}$ estimated by SILO falls in the Oracle optimized $\texttt{max\_lr}$ interval at each pruning cycle, meaning that the performance of SILO is competitive to the Oracle. Via this experiment on ResNet-20, we highlight the competitiveness of the proposed SILO again.


{\renewcommand{\arraystretch}{1.3}
	\begin{table}[!ht]
		\centering
		\setlength\tabcolsep{8.6pt}
		\begin{tabular}{l|c|c|c|c|c}
			\midrule
			\midrule
			{\small Percent of Weights Remaining, $\lambda$} & 100 & 51.3 & 41.1 & 32.9 & 21.1 \\ \midrule
			{\small Oracle \texttt{max\_lr} } & 3.4 & 3.8 & 4.6 & 5.6 & 6.2 \\
			{\small Oracle interval }  & {\small [2.8, 3.6]} & {\small [3.4, 4.2]} & {\small [3.8, 5.2]} & {\small [5.4, 6.6]} & {\small [5.4, 6.8]} \\
			{\small SILO \texttt{max\_lr}} & 3 & 3.2 & 4.7 & 6.4 & 6.9  \\
			\bottomrule
			\bottomrule
		\end{tabular}
		\vspace{-3mm}
		\caption{Comparison between Oracle tuned $\texttt{max\_lr}$, Oracle optimized $\texttt{max\_lr}$ interval (both obtained via grid search) and $\texttt{max\_lr}$ estimated by SILO when iteratively pruning ResNet-20 on the CIFAR-10 dataset using the global magnitude \cite{han2015learning}.}
		\label{lr_table_resnet_20}
\end{table}}

\newpage
\subsection{Experimental Results for More Values of $\lambda$}
\label{A4}

We note that, in Tables \ref{per1} - \ref{per5}, we only show the experimental results for some key values of $\lambda$. In this subsection, we show the results for more values of $\lambda$ in Tables \ref{per1_extra} - \ref{per5_extra}. The LR schedules (i.e., LR-warmup) from Table \ref{per1_extra} - \ref{per4_extra} are from \cite{frankle2018lottery}, \cite{frankle2020linear}, \cite{zhao2019variational}, \cite{chin2020towards} and \cite{renda2020comparing}, respectively.
The LR schedule (i.e., cosine decay) in Table \ref{per5_extra} is from \cite{dosovitskiy2020image}. The implementation details are provided in the top row of each table and the descriptions of each benchmark LR schedule are summarized in Table \ref{schedules}. It should be noted that, for the IMP method examined in this work, we rewind the unpruned weights to their values during training (e.g., epoch 6), in order to obtain a more stable subnetwork \cite{frankle2019stabilizing}.

Due to the width of these tables, we rotate them and present the results in the landscape style. We observe that the performance of SILO for other values of $\lambda$ still outperforms the selected LR schedule benchmarks. For example, in Table \ref{per3_extra}, we find that SILO achieves an improvement of 4.0\% at $\lambda$ = 5.72 compared to LR warmup.

{\renewcommand{\arraystretch}{0.85}
	\begin{table*}[!t]
		\centering
		\setlength\tabcolsep{15.8pt}
		\begin{tabular}{ll}
			\toprule
			\toprule
			{\bf Schedule} & {\bf Description } (Iters: Iterations) \\ \midrule
			LR decay (\texttt{a}, \texttt{b}) & linearly decay the value of LR from \texttt{a} over \texttt{b} Iters.    \\
			cyclical LR (\texttt{a}, \texttt{b}, \texttt{c}) & linearly vary between \texttt{a} and \texttt{b} with a step size of \texttt{c} Iters. \\
			LR warmup (\texttt{a}, \texttt{b}, \texttt{c}, \texttt{d}, \texttt{e}) & increase to \texttt{a} over \texttt{b} Iters, 10x drop at \texttt{c}, \texttt{d}, \texttt{e} Iters.  \\ \midrule
			SILO ($\epsilon$, $\delta$, \texttt{b}, \texttt{c}, \texttt{d}, \texttt{e}) & LR warmup ($\texttt{max\_lr}$, \texttt{b}, \texttt{c}, \texttt{d}, \texttt{e}), where $\texttt{max\_lr}$ increases  \\
			& from $\epsilon$ to $\epsilon + \delta$ during iterative pruning (see \eqref{SILO-eq}).   \\
			\bottomrule
			\bottomrule
		\end{tabular}
		\vspace{-3mm}
		\caption{Descriptions of LR schedule benchmarks and the proposed SILO.}
		\label{schedules}
		\vspace{-3mm}
\end{table*}}

\begin{landscape}
	
\vspace*{0.05mm}\begin{table}[!ht]
		\centering
		\setlength\tabcolsep{21.3pt}
		\begin{tabular}{l rrrrr}
			\toprule
			\toprule
			{\small {\bf Params}: 227K} & {\small {\bf Train Steps}: 63K Iters} & {\small {\bf Batch}: 128} & {\small {\bf Batch Norm}: Yes} & {\small {\bf Optimizer}: SGD} & {\small {\bf Rate}: 0.2}\\
			\bottomrule
		\end{tabular}
		
		\setlength\tabcolsep{10.9pt}
		\hspace*{0.2mm}\begin{tabular}{lcccccccc}
			{\small Percent of Weights Remaining, $\lambda$} & 100 & 64 & 40.9 & 32.8 & 26.2  & 13.4 & 8.59 & 5.72 \\ \midrule
			
			{\small constant LR} ({\scriptsize 1\texttt{e}-2}) &  90.4$\pm${\scriptsize 0.4} & 89.7$\pm${\scriptsize 0.5} & 
			88.9$\pm${\scriptsize 0.7} & 88.1$\pm${\scriptsize 0.9} & 87.5$\pm${\scriptsize 0.8} & 86.0$\pm${\scriptsize 0.9} & 82.8$\pm${\scriptsize 0.9}  & 79.1$\pm${\scriptsize 0.8}  \\
			
			{\small LR decay} ({\scriptsize 3\texttt{e}-2, 63K}) & 91.2$\pm${\scriptsize 0.4} & 91.0$\pm${\scriptsize 0.3} & 
			90.3$\pm${\scriptsize 0.5} & 89.8$\pm${\scriptsize 0.4} & 89.0$\pm${\scriptsize 0.7} & 87.7$\pm${\scriptsize 0.6} & 83.9$\pm${\scriptsize 0.6} & 79.8$\pm${\scriptsize 0.7}  \\
			
			{\small cyclical LR} ({\scriptsize 0, 3\texttt{e}-2, 8K}) &  90.8$\pm${\scriptsize 0.3} & 90.4$\pm${\scriptsize 0.5} &
			90.1$\pm${\scriptsize 0.4} & 89.7$\pm${\scriptsize 0.6} & 88.2$\pm${\scriptsize 0.7} & 87.6$\pm${\scriptsize 0.8} & 84.1$\pm${\scriptsize 0.6} & 80.3$\pm${\scriptsize 0.7} \\
			
			{\small LR-warmup} ({\scriptsize 3\texttt{e}-2, 20K, 20K, 25K, Nil}) & {\bf 91.7$\pm${\scriptsize 0.2}} & 91.5$\pm${\scriptsize 0.3} & 90.8$\pm${\scriptsize 0.5} & 90.3$\pm${\scriptsize 0.4} & 89.8$\pm${\scriptsize 0.6} & 88.2$\pm${\scriptsize 0.6} & 85.9$\pm${\scriptsize 0.9} & 81.2$\pm${\scriptsize 1.1} \\
			
			{\small SILO} ({\scriptsize 3\texttt{e}-2, 4\texttt{e}-2, 20K, 20K, 25K, Nil}) & 91.7$\pm${\scriptsize 0.2} & {\bf 91.9$\pm${\scriptsize 0.4}} & 
			{\bf 91.2$\pm${\scriptsize 0.6}} & {\bf 90.8$\pm${\scriptsize 0.5}} & {\bf 90.3$\pm${\scriptsize 0.4}} & {\bf 89.7$\pm${\scriptsize 0.5}} & {\bf 87.5$\pm${\scriptsize 0.8}} & {\bf 82.7$\pm${\scriptsize 1.2}} \\
			\bottomrule
			\bottomrule
		\end{tabular}
		\label{per_resnet_20}
		\vspace{-3mm}
		\caption{Performance comparison (averaged top-1 test accuracy $\pm$ std over 5 runs) of pruning ResNet-20 on CIFAR-10 dataset using the global magnitude pruning method \cite{Blalock2020}. LR-warmup is the standard implementation used in \cite{frankle2018lottery,frankle2020linear}. }
		\label{per1_extra}
	\end{table}

\vspace*{7.7mm}\begin{table}[!ht]
		\centering
		\setlength\tabcolsep{21.3pt}
		\begin{tabular}{l rrrrr}
			\toprule
			\toprule
			{\small {\bf Params}: 139M} & {\small {\bf Train Steps}: 63K Iters} & {\small {\bf Batch}: 128} & {\small {\bf Batch Norm}: Yes} & {\small {\bf Optimizer}: SGD} & {\small {\bf Rate}: 0.2}\\
			\midrule
		\end{tabular}
		
		\setlength\tabcolsep{10.9pt}
		\hspace*{0.2mm}\begin{tabular}{lcccccccc}
			{\small Percent of Weights Remaining, $\lambda$} & 100 & 64 & 40.9 & 32.8 & 26.2  & 13.4 & 8.59 & 5.72 \\ \midrule
			
			{\small constant LR} ({\scriptsize 8\texttt{e}-3}) & 91.3$\pm${\scriptsize 0.3} & 90.5$\pm${\scriptsize 0.5} & 
			89.5$\pm${\scriptsize 0.6} & 88.8$\pm${\scriptsize 0.6}  & 87.4$\pm${\scriptsize 0.7}& 85.8$\pm${\scriptsize 0.6} & 82.2$\pm${\scriptsize 1.4} & 73.7$\pm${\scriptsize 1.3}  \\
			
			{\small LR decay} ({\scriptsize 1\texttt{e}-2, 63K}) & 92.0$\pm${\scriptsize 0.5} & 90.9$\pm${\scriptsize 0.5} & 
			90.2$\pm${\scriptsize 0.5} & 89.4$\pm${\scriptsize 0.4} & 88.6$\pm${\scriptsize 0.5} & 87.4$\pm${\scriptsize 0.6} & 83.3$\pm${\scriptsize 0.8} & 75.4$\pm${\scriptsize 0.9}  \\
			
			{\small cyclical LR} ({\scriptsize 0, 3\texttt{e}-2, 15K})  & 92.3$\pm${\scriptsize 0.6} & 91.2$\pm${\scriptsize 0.6} & 
			90.4$\pm${\scriptsize 0.4} & 89.8$\pm${\scriptsize 0.5} & 89.1$\pm${\scriptsize 0.6} & 88.6$\pm${\scriptsize 0.8} & 83.7$\pm${\scriptsize 1.0} & 75.7$\pm${\scriptsize 1.2} \\
			
			{\small LR-warmup} ({\scriptsize 1\texttt{e}-1, 10K, 32K, 48K, Nil}) & 92.2$\pm${\scriptsize 0.3} & 91.3$\pm${\scriptsize 0.2} & 90.6$\pm${\scriptsize 0.4} & 90.2$\pm${\scriptsize 0.5} & 89.8$\pm${\scriptsize 0.8} & 89.2$\pm${\scriptsize 0.8} & 84.5$\pm${\scriptsize 0.9} & 76.5$\pm${\scriptsize 1.1} \\

			{\small SILO} ({\scriptsize 4\texttt{e}-2, 6\texttt{e}-2, 10K, 32K, 48K, Nil})  & {\bf 92.6$\pm${\scriptsize 0.4}} & {\bf 91.8$\pm${\scriptsize 0.6}} & {\bf 90.9$\pm${\scriptsize 0.5}} & {\bf 90.6$\pm${\scriptsize 0.6}} & {\bf 90.3$\pm${\scriptsize 0.6}} & {\bf 89.8$\pm${\scriptsize 0.9}} & {\bf 86.1$\pm${\scriptsize 0.8}} &{\bf 78.5$\pm${\scriptsize 1.0}} \\
			
			\bottomrule
			\bottomrule
		\end{tabular}
		\vspace{-3mm}
		\caption{Performance comparison (averaged top-1 test accuracy $\pm$ std over 5 runs) of pruning VGG-19 on CIFAR-10 dataset using the global gradient pruning method \cite{Blalock2020}. LR-warmup is the standard implementation used in \cite{frankle2018lottery,frankle2020linear,liu2018rethinking}.}
		\label{per2_extra}
	\end{table}

\vspace*{7.7mm}\begin{table}[!ht]
		\centering
		\setlength\tabcolsep{18.8pt}
		\begin{tabular}{l rrrrrr}
			\toprule
			\toprule
			{\small {\bf Params}: 1M} & {\small {\bf Train Steps}: 117K Iters} & {\small {\bf Batch}: 128} & {\small {\bf Batch Norm}: Yes} & {\small {\bf Optimizer}: SGD} & {\small {\bf Pruning Rate}: 0.2} \\
			\bottomrule
		\end{tabular}
		
		\setlength\tabcolsep{10.9pt}
		\hspace*{0.2mm}\begin{tabular}{lcccccccccccc}
			{\small Percent of Weights Remaining, $\lambda$} & 100 & 64 & 40.9 & 32.8 & 26.2  & 13.4 & 8.59 & 5.72 \\ \midrule
			
			{\small constant LR} ({\scriptsize 1\texttt{e}-2}) & 73.7$\pm${\scriptsize 0.4} & 72.8$\pm${\scriptsize 0.4} & 
			71.4$\pm${\scriptsize 0.5} & 70.3$\pm${\scriptsize 0.8} & 68.1$\pm${\scriptsize 0.7}& 63.5$\pm${\scriptsize 0.6} & 60.8$\pm${\scriptsize 1.1} & 59.1$\pm${\scriptsize 1.2}  \\
			
			{\small LR decay} ({\scriptsize 4\texttt{e}-2, 117K}) & 74.3$\pm${\scriptsize 0.3} & 73.5$\pm${\scriptsize 0.9} & 
			72.2$\pm${\scriptsize 0.4} & 71.2$\pm${\scriptsize 0.8} & 69.0$\pm${\scriptsize 0.6} & 64.7$\pm${\scriptsize 0.7} & 62.6$\pm${\scriptsize 1.2} & 60.3$\pm${\scriptsize 1.4} \\
			
			{\small cyclical LR} ({\scriptsize 0, 4\texttt{e}-2, 24K})  & 74.4$\pm${\scriptsize 0.4} & 73.2$\pm${\scriptsize 0.5} & 
			72.0$\pm${\scriptsize 0.3} & 70.9$\pm${\scriptsize 0.6} & 69.4$\pm${\scriptsize 0.6} & 65.1$\pm${\scriptsize 0.9} & 63.0$\pm${\scriptsize 1.1} & 60.8$\pm${\scriptsize 1.3} \\
			
			{\small LR-warmup} ({\scriptsize 12\texttt{e}-2, 58K, 58K, 92K, Nil}) & 74.6$\pm${\scriptsize 0.5} & 73.4$\pm${\scriptsize 0.6} & 72.3$\pm${\scriptsize 0.4} & 71.5$\pm${\scriptsize 0.7} & 69.6$\pm${\scriptsize 0.8} & 65.8$\pm${\scriptsize 0.9} & 63.9$\pm${\scriptsize 1.0} & 61.2$\pm${\scriptsize 0.9} \\
			
			{\small SILO} ({\scriptsize 7\texttt{e}-2, 5\texttt{e}-2, 58K, 58K, 92K, Nil})   & {\bf 75.0$\pm${\scriptsize 0.5}} & {\bf 74.1$\pm${\scriptsize 0.3}} & 
			{\bf 72.9$\pm${\scriptsize 0.6}} & {\bf 72.4$\pm${\scriptsize 0.7}} & {\bf 70.8$\pm${\scriptsize 0.8}} & {\bf 67.6$\pm${\scriptsize 1.2}} & {\bf 65.7$\pm${\scriptsize 1.1}} &{\bf 63.7$\pm${\scriptsize 1.0}} \\
			\bottomrule
			\bottomrule
		\end{tabular}
		\vspace{-3mm}
		\caption{Performance comparison (averaged top-1 test accuracy $\pm$ standard deviation over 5 runs) of pruning DenseNet-40 on CIFAR-100 dataset using LAMP. LR-warmup is adapted from the standard implementation used in \cite{zhao2019variational,huang2017densely}.}
		\label{per3_extra}
	\end{table}

\begin{table}[!ht]
		\centering
		\setlength\tabcolsep{18.8pt}
		\begin{tabular}{l rrrrrr}
			\toprule
			\toprule
			{\small {\bf Params}: 2.36M} & {\small {\bf Train Steps}: 78K Iters} & {\small {\bf Batch}: 128} & {\small {\bf Batch Norm}: Yes} & {\small {\bf Optimizer}: SGD} & {\small {\bf Pruning Rate}: 0.2} \\
			\bottomrule
		\end{tabular}
		
		\setlength\tabcolsep{11.2pt}
		\hspace*{0.2mm}\begin{tabular}{lcccccccccccc}
			{\small Percent of Weights Remaining, $\lambda$} & 100 & 64 & 40.9 & 32.8 & 26.2  & 13.4 & 8.59 & 5.72 \\ \midrule
			
			{\small constant LR} ({\scriptsize 1\texttt{e}-2}) & 72.7$\pm${\scriptsize 0.2} & 71.5$\pm${\scriptsize 0.4} & 
			70.4$\pm${\scriptsize 0.5} & 69.8$\pm${\scriptsize 1.1} & 68.2$\pm${\scriptsize 0.9}& 64.5$\pm${\scriptsize 0.6} & 63.8$\pm${\scriptsize 1.1} & 62.1$\pm${\scriptsize 1.2}  \\
			
			{\small LR decay} ({\scriptsize 15\texttt{e}-1, 78K}) & 73.3$\pm${\scriptsize 0.3} & 72.1$\pm${\scriptsize 0.3} & 
			71.8$\pm${\scriptsize 0.4} & 70.9$\pm${\scriptsize 1.0} & 69.4$\pm${\scriptsize 0.6} & 65.9$\pm${\scriptsize 0.7} & 65.1$\pm${\scriptsize 0.8} & 64.0$\pm${\scriptsize 1.1} \\
			
			{\small cyclical LR} ({\scriptsize 0, 5\texttt{e}-2, 14K})  & 73.5$\pm${\scriptsize 0.4} & 72.3$\pm${\scriptsize 0.5} & 
			72.0$\pm${\scriptsize 0.3} & 71.5$\pm${\scriptsize 0.7} & 69.6$\pm${\scriptsize 0.6} & 66.7$\pm${\scriptsize 0.9} & 65.3$\pm${\scriptsize 1.0} & 64.3$\pm${\scriptsize 1.2} \\
			
			{\small LR-warmup} ({\scriptsize 1\texttt{e}-1, 23K, 23K, 46K, 62K}) & 73.7$\pm${\scriptsize 0.4} & 72.5$\pm${\scriptsize 0.4} & 72.3$\pm${\scriptsize 0.5} & 72.1$\pm${\scriptsize 0.8} & 70.5$\pm${\scriptsize 0.9} & 67.3$\pm${\scriptsize 0.8} & 66.2$\pm${\scriptsize 1.1} & 64.8$\pm${\scriptsize 1.5} \\
			
			{\small SILO} ({\scriptsize 5\texttt{e}-2, 5\texttt{e}-2, 23K, 23K, 46K, 92K})   & {\bf 74.0$\pm${\scriptsize 0.5}} & {\bf 73.0$\pm${\scriptsize 0.3}} & {\bf 72.9$\pm${\scriptsize 0.8}} & {\bf 72.5$\pm${\scriptsize 0.6}}  & {\bf 71.0$\pm${\scriptsize 0.7}} & {\bf 68.8$\pm${\scriptsize 0.8}} & {\bf 67.6$\pm${\scriptsize 1.2}} & {\bf 66.8$\pm${\scriptsize 1.4}} \\
			\bottomrule
			\bottomrule
		\end{tabular}
		\vspace{-3mm}
		\caption{Performance comparison (averaged top-1 test accuracy $\pm$ standard deviation over 5 runs) of pruning MobileNetV2 on CIFAR-100 dataset using LAP. LR-warmup is adapted from the standard implementation used in  \cite{chin2020towards}.}
		\label{per3_extra_ex}

	\end{table}

	\begin{table}[!ht]
		\centering
		\setlength\tabcolsep{21.3pt}
		\begin{tabular}{l rrrrr}
			\toprule
			\toprule
			{\small {\bf Params}: 25.5M} & {\small {\bf Train Steps}: 70K Iters} & {\small {\bf Batch}: 128} & {\small {\bf Batch Norm}: Yes} &  {\small {\bf Optimizer}: SGD} & {\small {\bf Rate}: 0.2} \\
			\midrule
		\end{tabular}
		
		\setlength\tabcolsep{11.1pt}
		\hspace*{0.2mm}\begin{tabular}{lcccccccc}
			{\small Percent of Weights Remaining, $\lambda$} & 100 & 64 & 40.9 & 32.8 & 26.2  & 13.4 & 8.59 & 5.72 \\ \midrule
			
			{\small constant LR} ({\scriptsize 1\texttt{e}-2}) & 75.3$\pm${\scriptsize 0.2} & 75.2$\pm${\scriptsize 0.3} & 
			74.6$\pm${\scriptsize 0.7} & 74.2$\pm${\scriptsize 0.8}   & 73.9$\pm${\scriptsize 0.7} & 72.2$\pm${\scriptsize 0.9} & 70.5$\pm${\scriptsize 0.6} & 69.2$\pm${\scriptsize 0.9} \\
			
			{\small LR decay} ({\scriptsize 3\texttt{e}-2, 70K} & 76.5$\pm${\scriptsize 0.2} & 76.1$\pm${\scriptsize 0.5} & 
			75.8$\pm${\scriptsize 0.6}  & 75.6$\pm${\scriptsize 0.5}   & 75.1$\pm${\scriptsize 0.5} & 73.9$\pm${\scriptsize 0.7}  & 72.7$\pm${\scriptsize 0.8}  & 70.5$\pm${\scriptsize 0.6} \\
			
			{\small cyclical LR} ({\scriptsize 0, 5\texttt{e}-2, 20K}) & 76.8$\pm${\scriptsize 0.3} & 76.9$\pm${\scriptsize 0.5} & 77.0$\pm${\scriptsize 0.6}  & 76.5$\pm${\scriptsize 0.5}   & 75.5$\pm${\scriptsize 0.6} & 74.5$\pm${\scriptsize 0.6} & 73.4$\pm${\scriptsize 0.8} & 71.2$\pm${\scriptsize 0.7}  \\
			
			{\small LR-warmup} ({\scriptsize 1\texttt{e}-1, 4K, 23K, 46K, 62K}) & 77.0$\pm${\scriptsize 0.1} & 77.2$\pm${\scriptsize 0.2} & 76.9$\pm${\scriptsize 0.5} & 76.6$\pm${\scriptsize 0.2} &  75.8 $\pm${\scriptsize 0.3} & 75.2$\pm${\scriptsize 0.4} & 73.8$\pm${\scriptsize 0.5} & 71.5$\pm${\scriptsize 0.4} \\
			
			{\small SILO} ({\scriptsize 5\texttt{e}-2, 5\texttt{e}-2, 4K, 23K, 46K, 62K}) &  {\bf 77.2$\pm${\scriptsize 0.2}} & {\bf 77.4$\pm${\scriptsize 0.3}} & {\bf 77.0$\pm${\scriptsize 0.4}}  & {\bf 76.8$\pm${\scriptsize 0.4}}   & {\bf 76.1$\pm${\scriptsize 0.7}} & 
			{\bf 75.8$\pm${\scriptsize 0.6}} & {\bf 75.2$\pm${\scriptsize 0.8}} & {\bf 73.8$\pm${\scriptsize 0.6}} \\
			\bottomrule
			\bottomrule
		\end{tabular}
		\vspace{-3mm}
		\caption{Performance comparison (averaged top-1 test accuracy $\pm$ standard deviation over 5 runs) of pruning ResNet-50 on ImageNet dataset using the IMP pruning method \cite{tinyimagenet}. LR-warmup is adapted from the standard implementation used in \cite{frankle2020linear,renda2020comparing}.}
		\label{per4_extra}
	\end{table}

\begin{table}[!ht]
		\centering
		\setlength\tabcolsep{21.3pt}
		\begin{tabular}{l rrrrr}
			\toprule
			\toprule
			{\small {\bf Params}: 86M} & {\small {\bf Train Steps}: 2K Iters} & {\small {\bf Batch}: 1024} & {\small {\bf Batch Norm}: Yes} &  {\small {\bf Optimizer}: Adam} & {\small {\bf Rate}: 0.2} \\
			\midrule
		\end{tabular}
		
		\setlength\tabcolsep{9.8pt}
		\hspace*{0.2mm}\begin{tabular}{lcccccccc}
			{\small Percent of Weights Remaining, $\lambda$} & 100 & 64 & 40.9 & 32.8 & 26.2  & 13.4 & 8.59 & 5.72 \\ \midrule
			
			{\small constant LR} ({\scriptsize 5\texttt{e}-5}) & 97.4$\pm${\scriptsize 0.2} & 97.6 $\pm${\scriptsize 0.3} & 97.5 $\pm${\scriptsize 0.4} & 96.4 $\pm${\scriptsize 0.5} & 96.0$\pm${\scriptsize 0.7} & 86.3$\pm${\scriptsize 0.6} & 83.0$\pm${\scriptsize 0.9} & 80.1$\pm${\scriptsize 0.8} \\
			
			{\small cosine decay} ({\scriptsize 1\texttt{e}-4, 2K} & 98.0$\pm${\scriptsize 0.3} & 98.2 $\pm$ {\scriptsize 0.3} & 97.9 $\pm$ {\scriptsize 0.4} & 97.2 $\pm$ {\scriptsize 0.2} & 96.5$\pm${\scriptsize 0.6} & 87.7$\pm${\scriptsize 0.5} & 
			84.1$\pm${\scriptsize 1.0} & 81.6$\pm${\scriptsize 1.1}  \\
			
			{\small cyclical LR} ({\scriptsize 0, 1\texttt{e}-4, 2K}) & 97.8$\pm${\scriptsize 0.4} & 98.0 $\pm$ {\scriptsize 0.2} & 97.8 $\pm$ {\scriptsize 0.3} & 97.0 $\pm$ {\scriptsize 0.2} & 96.5$\pm${\scriptsize 0.6} & 87.2$\pm${\scriptsize 0.9} & 
			83.4$\pm${\scriptsize 0.6} & 81.0$\pm${\scriptsize 1.1} \\
			
			{\small LR-warmup} ({\scriptsize 4\texttt{e}-4, 0.3K, 0.5K, 0.9K, 1.3K}) & 98.0$\pm${\scriptsize 0.3} & 98.4 $\pm${\scriptsize 0.3} & 97.8 $\pm${\scriptsize 0.5} & 97.3 $\pm${\scriptsize 0.6} & 96.8$\pm${\scriptsize 0.7} & 88.1$\pm${\scriptsize 0.9} & 84.4$\pm${\scriptsize 0.8} & 82.1$\pm${\scriptsize 0.9} \\
			
			{\small SILO} ({\scriptsize 1\texttt{e}-4, 3\texttt{e}-4, 0.3K, 0.5K, 0.9K, 1.3K}) &  {\bf 98.0$\pm${\scriptsize 0.3}} & {\bf 98.5$\pm${\scriptsize 0.4}} & {\bf 98.0$\pm${\scriptsize 0.3}} & {\bf 97.7$\pm${\scriptsize 0.5}} & {\bf 97.4$\pm${\scriptsize 0.6}} & {\bf 89.2$\pm${\scriptsize 0.8}} & {\bf 85.5$\pm${\scriptsize 0.9}} & {\bf 83.4$\pm${\scriptsize 0.8}} \\
			\bottomrule
			\bottomrule
		\end{tabular}
		\vspace{-3mm}
		\caption{Performance comparison (averaged top-1 test accuracy $\pm$ standard deviation over 5 runs) of pruning pruning Vision Transformer (ViT-B-16) on CIFAR-10 using IMP \cite{frankle2018lottery}. Cosine decay is the standard implementation used in \cite{dosovitskiy2020image}.}
		\label{per5_extra}
		
	\end{table}
	
\end{landscape}

\end{document}